\documentclass{article}

\usepackage{microtype}
\usepackage{graphicx}
\usepackage{subfigure}
\usepackage{booktabs} 

\usepackage{amsmath,amsfonts,bm}

\def\eqref#1{equation~\ref{#1}}

\def\1{\bm{1}}

\DeclareMathAlphabet{\mathsfit}{\encodingdefault}{\sfdefault}{m}{sl}
\SetMathAlphabet{\mathsfit}{bold}{\encodingdefault}{\sfdefault}{bx}{n}

\def\gG{{\mathcal{G}}}

\def\sN{{\mathbb{N}}}

\def\sR{{\mathbb{R}}}

\def\sY{{\mathbb{Y}}}

\newcommand{\R}{\mathbb{R}}

\usepackage{hyperref}

\usepackage[accepted]{icml2024}

\usepackage{amsmath}
\usepackage{amssymb}
\usepackage{mathtools}
\usepackage{amsthm}

\usepackage[capitalize,noabbrev]{cleveref}

\theoremstyle{plain}
\newtheorem{theorem}{Theorem}[section]
\newtheorem{proposition}[theorem]{Proposition}
\newtheorem{lemma}[theorem]{Lemma}

\theoremstyle{definition}
\newtheorem{definition}[theorem]{Definition}

\theoremstyle{remark}

\newcommand{\msl}{\{\!\!\{}
\newcommand{\msr}{\}\!\!\}}
\usepackage{xcolor,colortbl}
\usepackage{tabularx}
\newcommand{\countablecolor}{\cellcolor{yellow}}
\usepackage{enumitem}
\usepackage[textsize=tiny]{todonotes}

\icmltitlerunning{Graph As Point Set}

\begin{document}

\twocolumn[
\icmltitle{Graph as Point Set}



\icmlsetsymbol{equal}{*}
\begin{icmlauthorlist}
\icmlauthor{Xiyuan Wang}{iai}
\icmlauthor{Pan Li}{git}
\icmlauthor{Muhan Zhang}{iai}
\end{icmlauthorlist}

\icmlaffiliation{iai}{Institute for Artificial Intelligence, Peking University}
\icmlaffiliation{git}{Georgia Institute of Technology}

\icmlcorrespondingauthor{Xiyuan Wang}{wangxiyuan@pku.edu.edu}
\icmlcorrespondingauthor{Muhan Zhang}{muhan@pku.edu.cn}
\icmlkeywords{Graph Neural Network}

\vskip 0.3in
]




\begin{abstract}
Graph is a fundamental data structure to model interconnections between entities. Set, on the contrary, stores independent elements. To learn graph representations, current Graph Neural Networks (GNNs) primarily use message passing to encode the interconnections. In contrast, this paper introduces a novel graph-to-set conversion method that bijectively transforms interconnected nodes into a set of independent points and then uses a set encoder to learn the graph representation. This conversion method holds dual significance. Firstly, it enables using set encoders to learn from graphs, thereby significantly expanding the design space of GNNs. Secondly, for Transformer, a specific set encoder, we provide a novel and principled approach to inject graph information losslessly, different from all the heuristic structural/positional encoding methods adopted in previous graph transformers. To demonstrate the effectiveness of our approach, we introduce Point Set Transformer (PST), a transformer architecture that accepts a point set converted from a graph as input. Theoretically, PST exhibits superior expressivity for both short-range substructure counting and long-range shortest path distance tasks compared to existing GNNs. Extensive experiments further validate PST's outstanding real-world performance. Besides Transformer, we also devise a Deepset-based set encoder, which achieves performance comparable to representative GNNs, affirming the versatility of our graph-to-set method.
\end{abstract}
\section{Introduction}
Graph, composed of interconnected nodes, has a wide range of applications and has been extensively studied. In graph machine learning, a central focus is to effectively leverage node connections. Various architectures have arisen for graph tasks, exhibiting significant divergence in their approaches to utilizing adjacency information.

Two primary paradigms have evolved for encoding adjacency information. The first paradigm involves message passing between nodes via edges. Notable methods in this category include Message Passing Neural Network (MPNN)~\citep{MPNN}, a foundational framework for GNNs such as GCN~\citep{GCN}, GIN~\citep{GIN}, and GraphSAGE~\citep{SAGE}. Subgraph-based GNNs~\citep{NGNN,I2GNN,ESAN,OSAN,SUN,GNNAK,SSWL} select subgraphs from the whole graph and run MPNN within each subgraph. These models aggregate messages from neighbors to update the central nodes' representations. Additionally, Graph Transformers (GTs) integrate adjacency information into the attention matrix~\citep{GT-graphit, GT-SAN, GT-graphtrans, firstGT, GT-graphormer, GT-Exphormer} (note that some early GTs have options to not use adjacency matrix by using only positional encodings, but the performance is significantly worse~\citep{firstGT}). Some recent GTs even directly incorporate message-passing layers into their architectures~\citep{GT-GPS,GT-DeepSet}. In summary, this paradigm relies on adjacency relationships to facilitate information exchange among nodes.

The second paradigm designs permutation-equivariant neural networks that directly take adjacency matrices as input. This category includes high-order Weisfeiler-Leman tests~\citep{PPGN}, invariant graph networks~\citep{IGN}, and relational pooling~\citep{CountSubg}. Additionally, various studies have explored manual feature extraction from the adjacency matrix, including random walk structural encoding~\citep{RWSE,DE}, Laplacian matrix eigenvectors~\citep{EquiStableEnc,SignBasis-inv,huang2023stability}, and shortest path distances~\citep{DE}. However, these approaches typically serve as data augmentation steps for other models, rather than constituting an independent paradigm.

Both paradigms heavily rely on adjacency information in graph encoding. In contrast, this paper explores whether we can give up adjacency matrix in graph models while achieving competitive performance. As shown in Figure~\ref{fig::G2s}, our innovative graph-to-set method converts interconnected nodes into independent points, subsequently encoded by a set encoder like Transformer. Leveraging our \textit{symmetric rank decomposition}, we break down the augmented adjacency matrix $A+D$ into $QQ^T$, wherein $Q$ is constituted by column-full-rank rows, each denoting a node \textit{coordinate}. This representation enables us to express the presence of edges as inner products of coordinate vectors ($Q_i$ and $Q_j$). Consequently, interlinked nodes can be transformed into independent points and supplementary coordinates without information loss. Theoretically, two graphs are isomorphic iff the two converted point sets are equal up to an \textit{orthogonal transformation} (because for any $QQ^T=A+D$, $QR$ is also a solution where $R$ is any orthogonal matrix). This equivalence empowers us to encode the set with coordinates in an orthogonal-transformation-equivariant manner, akin to E(3)-equivariant models designed for 3D geometric deep learning. Importantly, our approach is versatile, allowing for using any equivariant set encoder, thereby significantly expanding the design space of GNNs. Furthermore, for Transformer, a specific set encoder, our method offers a novel and principled way to inject graph information losslessly. In Appendix~\ref{app::SE}, we additionally show that it unifies various heuristic structural/positional encodings in previous GTs, including random walk~\citep{DE,benchmarkingGNN,GT-GPS}, heat kernel~\citep{GT-graphit}, and resistance distance~\citep{biconnectivity}.

\begin{figure}
\centering
\includegraphics[width=0.45\textwidth]{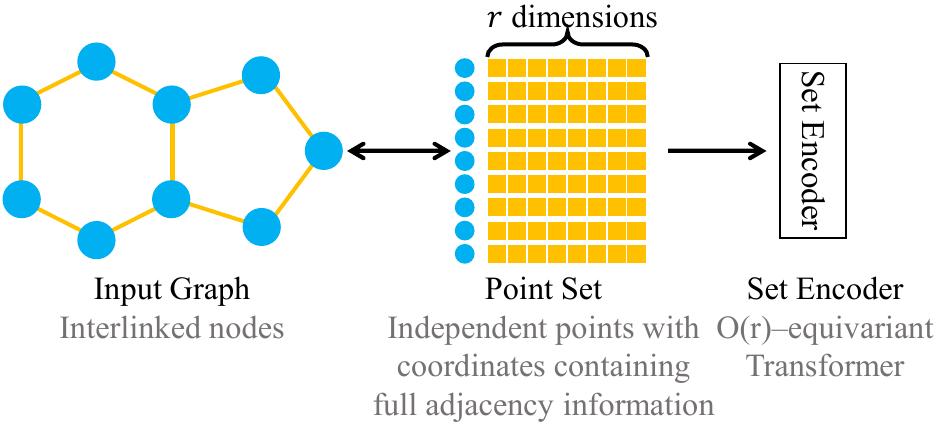}
\vskip -10pt
\caption{Our method converts the input graph to a point set first and encoding it with a set encoder. $O(r)$ denotes the set of $r$-dimension orthogonal transformations.}\label{fig::G2s}
\vskip -20pt
\end{figure}
To instantiate our method, we introduce an orthogonal-transformation-equivariant Transformer, namely Point Set Transformer (PST), to encode the point set. PST provably surpasses existing models in long-range and short-range expressivity. Extensive experiments verify these claims across synthetic datasets, graph property prediction datasets, and long-range graph benchmarks. Specifically, PST outperforms all baselines on QM9~\citep{QM9} dataset. Moreover, our graph-to-set method is not constrained to only one specific set encoder. We also propose a Deepset~\citep{DeepSet}-based model, which outperforms comparable to GIN~\citep{HowPowerfulAreGNNs} on our datasets.

\textbf{Differences from eigendecomposition.}~Note that our graph-to-set method is distinct from previous approaches that decompose adjacency matrices for positional encodings~\citep{benchmarkingGNN,EquiStableEnc,SignBasis-inv,Specformer}. The key differences root in that previous methods primarily relied on eigendecomposition (EVD), whereas our method is based on symmetric rank decomposition (SRD). Their differences are as follows:
\begin{itemize}[itemsep=2pt,topsep=-2pt,parsep=0pt,leftmargin=10pt]
\item SRD enables a practical conversion of graph problems into set problems. SRD of a matrix is unique up to a single orthogonal transformation, while EVD is unique up to a combination of orthogonal transformations within each eigenspace. This difference allows SRD-based models to easily maintain symmetry, ensuring consistent predictions for isomorphic graphs, while EVD-based methods~\citep{SignBasis-inv} struggle because they need to deal with each eigenspace individually, making them less suitable for graph-level tasks where eigenspaces vary between graphs.
\item Due to the advantage of SRD, we can utilize set encoder with  coordinates to capture graph structure, thus expanding the design space of GNN. Moreover, our method provides a principled way to add graph information to Transformers. Note that previous GTs usually require multiple heuristic encodings together. Besides node positional encodings, they also use adjacency matrices: Grit~\citep{Grit} and graphit~\citep{GT-graphit} use random walk matrix (normalized adjacency) as relative positional encoding (RPE). Graph Transformer~\citep{firstGT}, Graphormer~\citep{GT-graphormer}, and SAN~\citep{GT-SAN} use adjacency matrix as RPE. \citet{firstGT}'s ablation shows that adjacency is crucial. GPS~\citep{GT-GPS}, Exphormer~\citep{GT-Exphormer}, higher-order Transformer~\citep{GT-DeepSet}, and GraphVit/MLP-Mixer~\citep{GraphVit} even directly incorporate message passing blocks which use adjacency matrix to guide message passing between nodes.
\end{itemize}
In summary, this paper introduces a novel approach to graph representation learning by converting interconnected graphs into independent points and subsequently encoding them using an orthogonal-transformation-equivariant set encoder like our Point Set Transformer. This innovative approach outperforms existing methods in both long- and short-range tasks, as validated by comprehensive experiments.

\vspace{-2mm}
\section{Preliminary}\label{sec:preliminary}
For a matrix $Z\!\in\!\sR^{a\times b}$, we define $Z_i\!\in\!\sR^b$ as the $i$-th row (as a column vector), and $Z_{ij}\!\in\!\sR$ as its $(i,j)$ element. For a vector $\Lambda\!\in\!\sR^a$, $\text{diag}(\Lambda)\!\in\!\sR^{a\times a}$ is the diagonal matrix with $\Lambda$ as its diagonal elements. And for $S\!\in\!\sR^{a\times a}$, $\text{diagonal}(S)\!\in\!\sR^{a}$ represents the vector of its diagonal elements.

Let $\mathcal{G} = (V, E, X)$ denote an \textit{undirected graph}. Here, $V = \{1, 2, 3, ..., n\}$ is the set of $n$ nodes, $E \subseteq V \times V$ is the set of edges, and $X \in \mathbb{R}^{n \times d}$ is the node feature matrix, whose $v$-th row $X_v$ is of node $v$. The edge set $E$ can also be represented using the adjacency matrix $A \in \mathbb{R}^{n \times n}$, where $A_{uv}$ is $1$ if the edge exists (i.e., $(u, v)\in E$) and $0$ otherwise. A graph $\mathcal{G}$ can also be represented by the pair $(V, A, X)$ or $(A, X)$. The degree matrix $D$ is a diagonal matrix with node degree (sum of a row of matrix $A$) as the diagonal elements. 

Given a permutation function $\pi\!:\!\{1, 2, 3, ..., n\}\!\to\!\{1, 2, 3, ..., n\}$, the permuted graph is $\pi(\mathcal{G}) = (\pi(A), \pi(X))$, where $\pi(A) \in \mathbb{R}^{n \times n}, \pi(A)_{\pi(u) \pi(v)} = A_{uv}$, and $\pi(X) \in \mathbb{R}^{n \times d}, \pi(X)_{\pi(v)} = X_v$ for all $u,v\in V$. Essentially, the permutation $\pi$ reindex each node $v$ to $\pi(v)$ while preserving the original graph structure and node features. Two graphs are isomorphic iff they can be mapped to each other through a permutation.

\begin{definition}
Graphs $\mathcal{G}_1 = (A_1, X_1)$ and $\mathcal{G}_2 = (A_2, X_2)$ are isomorphic, denoted as $\mathcal{G}_1 \simeq \mathcal{G}_2$, if there exists a permutation $\pi$ such that $\pi(A_1) = A_2$ and $\pi(X_1) = X_2$.
\end{definition} 
Isomorphic graphs can be transformed into each other by merely reindexing their nodes. In graph tasks, models should assign the same prediction to isomorphic graphs.

\textbf{Symmetric Rank Decomposition (SRD).}~Decomposing an matrix into two full-rank matrices is well-known~\citep{puntanen2011matrix}. We further show that a positive semi-definite matrix can be decomposed into a full-rank matrix.
\begin{definition}
(Symmetric Rank Decomposition, SRD) Given a (symmetric) positive semi-definite matrix $L \in \mathbb{R}^{n \times n}$ of rank $r$, its SRD is $Q \in \mathbb{R}^{n \times r}$, where $L\!=\!QQ^T$.
\end{definition}
As $L = QQ^T$, $rank(Q)\!=\!rank(L)\!=\!r$, which implies that $Q$ must be full column rank. Moreover, two SRDs of the same matrix are equal up to an orthogonal transformation. Let $O(r)$ denote the set of orthogonal matrices in $\mathbb{R}^{r \times r}$.
\begin{proposition}\label{prop::SRD}
Matrices $Q_1$ and $Q_2$ in $\mathbb{R}^{n \times r}$ are SRD of the same matrix iff there exists $R \in O(r), Q_1 = Q_2 R$.
\end{proposition}
SRD is closely related to eigendecomposition. Let $L = U \text{diag}(\Lambda) U^T$ denote the eigendecomposition of $L$, where $\Lambda\!\in\!\mathbb{R}^{r}$ is the vector of non-zero eigenvalues, and $U\!\in\!\mathbb{R}^{n \times r}$ is the matrix whose columns are the corresponding eigenvectors. $Q = U \text{diag}(\Lambda^{1/2})$ yields an SRD of $L$, where the superscript denotes element-wise square root operation.
\section{Graph as Point Set}\label{sec:g2p}
In this section, we present our innovative method for converting graphs into sets of points. We first show that Symmetric Rank Decomposition (SRD) can theoretically achieve this transformation: two graphs are isomorphic iff the sets of coordinates generated by SRD are equal up to orthogonal transformations. Additionally, we parameterize SRD for better real-world performance. Proof details are in Appendix~\ref{app::proof}.
\subsection{Symmetric Rank Decomposition for Coordinates}\label{sec::g2s}

A natural approach to breaking down the interconnections between nodes is to decompose the adjacency matrix. While previous methods often used eigendecomposition outputs as supplementary node features, these features are not unique. Consequently, models relying on them fail to provide consistent predictions for isomorphic graphs, ultimately leading to poor generalization. To address this, we show that Symmetric Rank Decomposition (SRD) can convert graph-level tasks into set-level tasks with perfect alignment. Since SRD only applies to positive semi-definite matrices, we use the augmented adjacency matrix $D + A$, which is always positive semi-definite (proof in Appendix~\ref{app::proof::d+a}).

\begin{theorem}\label{thm::g2s}
Given two graphs $\mathcal{G}=(V, A,X)$ and $\mathcal{G'}=(V',A',X')$ with respective degree matrices $D$ and $D'$, $\mathcal{G}\simeq \mathcal{G'}$ iff $\exists R\in O(r), \msl (X_{v}, RQ_{v})|\forall v \in V\msr = \msl (X'_{v}, Q_{v}')|v\in V'\msr$, where $Q$ and $Q'$ are the SRD of $D+A$ and $D'+A'$ respectively, and $r$ is the rank of $Q$.
\end{theorem}
In this theorem, the graph $\mathcal{G}=(V, A,X)$ is converted to a set of points $\{(X_{v}, Q_{v})|v\in V\}$, where $X_v$ is the original node feature of $v$, and $Q_v$, the $v$-th row of SRD of $D+A$, is the $r$-dimensional coordinate of node $v$. Consequently, two graphs are isomorphic iff their point sets are equal up to an orthogonal transformation. Intuitively, we can imagine that the graph is mapped into an $r$-dimensional space, where each node has a coordinate, and the inner product between two coordinates represents edge existence. This mapping is not unique, since we can freely rotate the coordinates through an orthogonal transformation without changing inner products. This conversion can be loosely likened to the reverse process of constructing molecular graph from atoms' 3D coordinates, where Euclidean distances between atoms determine node connections in the graph. 

Leveraging Theorem~\ref{thm::g2s}, we can convert a graph into a set and employ a set encoder for encoding it. Our method consistently produces representations for isomorphic graphs when the encoder is orthogonal transformation-invariant. The method's expressivity hinges on the set encoder's ability to differentiate non-equal sets, with greater encoder power enhancing overall performance on graph tasks.
\subsection{Parameterized Coordinates}\label{sec::psrd}
In this section, we enhance SRD's practical performance through parameterization. As shown in Section~\ref{sec:preliminary}, SRD can be implemented via eigendecomposition: $Q=U\text{diag}(\Lambda^{1/2})$, where $\Lambda\in \sR^{r}$ denotes non-zero eigenvalues of the decomposed matrix, and $U\in \R^{n\times r}$ denotes corresponding eigenvectors. To parameterize SRD, we replace the element-wise square root with a function $f:\sR^r\to \sR^r$. This alteration further eliminates the constraint of non-negativity on eigenvalues and enables the use of various symmetric matrices containing adjacency information to generate coordinates. Additionally, for model flexibility, the coordinates can include multiple channels, with each channel corresponding to a distinct eigenvalue function.

\begin{definition}
(Parameterized SRD, PSRD)~With a $d$-channel eigenvalue function $\!f\!:\!\sR^r\!\to\!\sR^{r\times d}\!$ and an adjacency function $\mathcal{Z}\!:\!\sR^{n\times n}\!\to\!\sR^{n\times n}$ producing symmetric matrices, PSRD coordinate of a graph $\gG\!=\!(V, A, X)$ is $\mathcal{Q}(\!\mathcal{Z}(A), f\!)\!\in\!\sR^{n\times r\times d}$, whose $i$-th channel is $Udiag(f_i(\Lambda))\!\in\!\sR^{n\times r}$, where $\Lambda\!\in\!\sR^{r}$, $U\!\in\!\sR^{n\times r}$ are non-zero eigenvalues and corresponding eigenvectors of $\mathcal{Z}(\!A\!)$, and $f_i\!:\!\sR^r\!\to\!\sR^r$ is the $i$-th channel of $f$.
\end{definition}

In the definition, $\mathcal{Z}$ maps adjacency matrix to its variants like Laplacian matrix, and $f$ transforms eigenvalues. $\mathcal{Q}(\mathcal{Z}(A), f)_u\in \sR^{r\times d}$ is node $u$'s coordinate. Similar to SRD, PSRD can also convert the graph isomorphism problems to set equality problems.
\begin{theorem}\label{thm::learng2s}
Given a permutation-equivariant adjacency function {$\mathcal Z$}, for graphs $\gG\!=\!(V, A,X)$ and $\gG'\!=\!(V', A',X')$ 
\begin{itemize}[itemsep=2pt,topsep=-2pt,parsep=0pt,leftmargin=10pt]
    \item If eigenvalue function $f$ is permutation-equivariant and $\gG\simeq \gG'$, then two point sets with PSRD coordinates are equal up to an orthogonal transformation, i.e.,
    $\exists R\!\in\!O(r)$, $\msl\!X\!_{v},\!R\mathcal{Q}(\mathcal Z(\!A\!),\!f)_{v}|v\!\in\!V\!\msr\!\!=
\!\!\msl\!X\!_{v}',\!\mathcal{Q}(\mathcal Z(\!A'\!),\!f)_{v}|v\!\in\!V\!'\!\msr$, where $r$ is the rank of coordinates.
    \item If $\mathcal{Z}$ is injective, for all $d\!\ge\!2$, there exists a continuous permutation-equivariant function $f\!:\!\sR^r\!\to\!\sR^{r\times d}$ that if two point sets with PSRD coordinates are equal up to an orthogonal transformation, $\gG\simeq \gG'$.
\end{itemize}
\end{theorem}

Given permutation equivariant $f$ and $\mathcal{Z}$, the point sets with PSRD coordinates are equal up to an orthogonal transformation for isomorphic graphs. Moreover, there exists $f$ making reverse true. Therefore, we can safely employ permutation-equivariant eigenvalue functions, ensuring consistent predictions for isomorphic graphs. An expressive eigenvalue function also allows for the lossless conversion of graph-level tasks into set problems. In implementation, we utilize DeepSet~\citep{DeepSet} due to its universal expressivity for permutation-equivariant set functions. Detailed architecture is shown in Figure~\ref{fig:g2sarch} in Appendix~\ref{app::arch}.

In summary, we use SRD and its parameterized generalization to decompose the adjacency matrix or its variants into coordinates. Thus, we transform a graph into a point set where each point represents a node and includes both the original node feature and the coordinates as its features.
\section{Point Set Transformer}\label{sec:pst}

Our method, as depicted in Figure~\ref{fig::G2s}, comprises two steps: converting the graph into a set of independent points and encoding the set. Section~\ref{sec:g2p} demonstrates the bijective transformation of the graph into a set. To encode this point set, we introduce a novel architecture, Point Set Transformer (PST), designed to maintain orthogonality invariance and deliver remarkable expressivity. Additionally, to highlight our method's versatility, we propose a DeepSet~\citep{DeepSet}-based set encoder in Appendix~\ref{app::PSDS}. 

PST's architecture is depicted in Figure~\ref{fig:arch} in Appendix~\ref{app::arch}. PST operates with two types of representations for each point: scalars, which remain invariant to coordinate orthogonal transformations, and vectors, which adapt equivariantly to coordinate changes. For a point $i$, its scalar representation is $s_i\!\in\!\sR^{d}$, and its vector representation is $v_i\!\in\!\sR^{r\times d}$, where $d$ is the hidden dimension, and $r$ is the rank of coordinates. $s_i$ and $v_i$ are initialized with the input node feature $X_i$ and PSRD coordinates (detailed in Section~\ref{sec::psrd}) containing graph structure information, respectively.

Similar to conventional transformers, PST comprises multiple layers. Each layer incorporates two key components:

\textbf{Scalar-Vector Mixer.}~This component, akin to the feed-forward network in Transformer, individually transforms point features. To enable information exchange between vectors and scalars, we employ the following architecture.
\begin{align}
s_i' &\leftarrow \text{MLP}_1(s_i \Vert \text{diagonal}(W_1v_i^Tv_iW_2^T)), \\
v_i' &\leftarrow v_i \text{diag}(\text{MLP}_2(s_i))W_3 + v_iW_4
\end{align}
Here, $W_1, W_2, W_3,$ and $W_4 \in \mathbb{R}^{d\times d}$ are learnable matrices for mixing different channels of vector features. Additionally, $\text{MLP}_1:\mathbb{R}^{2d\to d}$ and $\text{MLP}_2:\mathbb{R}^{d\to d}$ represent two multi-layer perceptrons transforming scalar representations. The operation $\text{diagonal}(W_1v_i^Tv_iW_2)$ takes the diagonal elements of a matrix, which translates vectors to scalars, while $v_i\text{diag}(\text{MLP}_2(s_i))$ transforms scalar features into vectors. As $v_i^TR^TR v_i = v_i^T v_i, \forall R \in O(r)$, the scalar update is invariant to orthogonal transformations of the coordinates. Similarly, the vector update is equivariant to $O(r)$.

\textbf{Attention Layer.}~Akin to ordinary attention layers, this component compute pairwise attention score to linearly combine point representations. 
\begin{small}
\begin{equation}
\text{Atten}_{ij} = \text{MLP}((W_q^ss_i \odot W_k^ss_j) \Vert \text{diagonal}(W_q^vv_i^Tv_jW_k^v))
\end{equation}
\end{small}

Here, $W_q^s$ and $W_q^v$ denote the linear transformations for scalars and vectors queries, respectively, while $W_k^s$ and $W_k^v$ are for keys. The equation computes the inner products of queries and keys, similar to standard attention mechanisms. It is easy to see $\text{Atten}_{ij}$ is also invariant to $O(r)$.

Then we linearly combine point representations with attention scores as the coefficients:
\begin{small}
\begin{equation}
s_i \leftarrow \sum_{j} \text{Atten}_{ij}s_j',\quad v_i \leftarrow \sum_{j} \text{Atten}_{ij}v_{j}'\end{equation}
\end{small}
\!Each transformer layer is of time complexity $O(n^2r)$ and space complexity $O(n^2+nr)$. 

\textbf{Pooling.}~After several layers, we pool all points' scalar representations as the set representation $s$.
\begin{equation}
    s\leftarrow\text{Pool}(\{s_i|i\in V\}),
\end{equation}
where $\text{Pool}$ is pooling function like sum, mean, and max.
\section{Expressivity}
In this section, we delve into the theoretical expressivity of our methods. Our PSRD coordinates and the PST architecture exhibit strong long-range expressivity, allowing for efficient computation of distance metrics between nodes, as well as short-range expressivity, enabling the counting of paths and cycles rooted at each node. Therefore, our model is more expressive than many existing models, including GIN (equivalent to the 1-WL test)~\citep{HowPowerfulAreGNNs}, PPGN (equivalent to the 2-FWL test, more expressive in some cases)~\citep{PPGN}, GPS~\citep{GT-GPS}, and Graphormer~\citep{GT-graphormer} (two representative graph transformers). More details are in Appendix~\ref{app::expressivity}.

\subsection{Long Range Expressivity}\label{sec::g2p_expressivity}
This section demonstrates that the inner products of PSRD coordinates exhibits strong long-range expressivity, which PST inherits by utilizing inner products in attention layers.

When assessing a model's capacity to capture long-range interactions (LRI), a key measure is its ability to compute shortest path distance (spd) between nodes. Since formally characterizing LRI can be challenging, we focus on analyzing models' performance concerning this specific measure. We observe that existing models vary significantly in their capacity to calculate spd. Moreover, we find an intuitive explaination for these differences: spd between nodes can be expressed as $spd(i, j, A) = \arg\min_{k}\{k|A^k_{ij}>0\}$, and the ability to compute $A^K$, the $K$-th power of the adjacency matrix $A$, can serve as a straightforward indicator. Different models need different number of layers to compute $A^K$.

\textbf{PSRD coordinates.}~PSRD coordinates can capture arbitrarily large shortest path distances through their inner products in one step. To illustrate it, we decompose the adjacency matrix as $A= U\text{diag}(\Lambda) U^T$, and employ coordinates as $U$ and $U\text{diag}(\Lambda^K)$. Their inner products are as follows:
\begin{small}
\begin{equation}
    \overbrace{U\text{diag}(\Lambda^K)U^T\rightarrow A^K}^{1~~ \text{step}}
\end{equation}
\end{small}
\begin{theorem}\label{thm::longrange_pst}
    There exists permutation-equivariant functions $f_k, k=0,1,2,...,K$, such that for all graphs $\gG=(A, X)$, the shortest path distance between node $i, j$ is a function of $\langle \mathcal{Q}(A, f_0)_i,\mathcal{Q}(A, f_k)_j\rangle$, $k\!=\!0,1,2,...K$, where $\mathcal{Q}(A, f)$ is the PSRD coordinate defined in Section~\ref{sec::psrd}, $K$ is the maximum shortest path distance between nodes. 
\end{theorem}
\textbf{2-FWL.}~A powerful graph isomorphic test, 2-Folklore-Weisfeiler-Leman Test (2-FWL), and its neural network version PPGN~\citep{PPGN} produce node pair representations in a matrix $X\in \sR^{n\times n}$. $X$ is initialized with $A$. Each layer updates $X$ with $XX$. So intuitively, computing $A^K$ takes $\lceil\log_2 K\rceil$ layers.
\begin{small}
\begin{equation}
    \overbrace{A\rightarrow A^2\!\!=\!\!AA\rightarrow A^4\!\!=\!\!A^2A^2\rightarrow ...\rightarrow A^K\!\!=\!\!A^{K/2}A^{K/2}}^{\lceil\log_2 K\rceil ~~\text{layers}}
\end{equation}\end{small}

\begin{theorem}\label{thm::2-FWL-spdlimit}
Let $c^{k}(\gG)_{ij}$ denote the color of node tuple $(i,j)$ of graph $\gG$ at iteration $k$. Given graphs $\gG\!=\!(A, X)$ and $\gG'\!=\!(A',X')$, for all $K\!\in\!\sN^+$, if two node tuples $(i,j)$ in $\gG$ and $(i',j')$ in $\gG'$ have $spd(i,j,A)\!<\!spd(i',j',A')\!\le\!2^K$, then $c^{K}(\gG)_{ij}\!\neq\!c^{K}(\gG')_{i'j'}$. Moreover, for all $L\!> \!2^K$, there exists $i,j,i',j'$, such that $spd(i,j,A)\!>\!spd(i',j',A')\!\ge\!L$  while $c^{K}(\gG)_{ij}\!=\!c^{K}(\gG')_{i'j'}$. 
\end{theorem}
In other words, $K$ iterations of 2-FWL can distinguish pairs of nodes with different spds, as long as that distance is at most $2^K$. Moreover, $K$-iteration 2-FWL cannot differentiate all tuples with spd $> 2^K$ from other tuples with different spds, which indicates that $K$-iteration 2-FWL is effective in counting shortest path distances up to a maximum of $2^K$.

\textbf{MPNN.}~Intuitively, each MPNN layer uses $AX$ to update node representations $X$. However, this operation in general cannot compute $A^K$ unless the initial node feature $X=I$. 
\begin{small}\begin{equation}
        \overbrace{X\rightarrow AX\rightarrow A^2X\!\!=\!\!AAX\rightarrow ... \to A^KX\!\!=\!\!AA^{K-1}X}^{K ~~\text{layers}}
\end{equation}\end{small}
\begin{theorem}\label{thm::longrange_mpnn_gt}
    A graph pair exists that MPNN cannot differentiate, but their sets of all-pair spd are different. 
\end{theorem}
If MPNNs can compute spd between node pairs, they should be able to distinguish this graph pair from the sets of spd. However, we show no MPNNs can distinguish the pair, thus proving that MPNNs cannot compute spd.

Graph Transformers (GTs) are known for their strong long-range capacities~\citep{LRGB}, as they can aggregate information from the entire graph to update each node's representation. However, aggregating information from the entire graph is not equivalent to capturing the distance between nodes, and some GTs also fail to compute spd between nodes. Details are in Appendix~\ref{app::GT_fail_spd}. Note that this slightly counter-intuitive results is because we take a new perspective to study long range interaction rather than showing GTs are weak in long range capacity.

Besides shortest path distances, our PSRD coordinates also enables the unification of various structure encodings (distance metrics between nodes), including random walk~\citep{DE,benchmarkingGNN,GT-GPS}, heat kernel~\citep{GT-graphit}, resistance distance~\citep{NGNN,biconnectivity}. Further insights and details are shown in Table~\ref{tab:parameter} in Appendix~\ref{app::SE}.

\subsection{Short Range Expressitivity}
This section shows PST's expressivity in representative short-range tasks: path and cycle counting.

\begin{theorem}\label{thm::count_path}
A one-layer PST can count paths of length $1$ and $2$, a two-layer PST can count paths of length $3$ and $4$, and a four-layer PST can count paths of length $5$ and $6$. Here, ``count'' means that the $(i,j)$ element of the attention matrix in the last layer can express the number of paths between nodes $i$ and $j$.
\end{theorem}
Therefore, with enough layers, our PST models can count the number of paths of length $\le 6$ between nodes. Furthermore, our PST can also count cycles.
\begin{theorem}\label{thm::count_cycle}
A one-layer PST can count cycles of length $3$, a three-layer PST can count cycles of length $4$ and $5$, and a five-layer PST can count cycles of length $6$ and $7$. Here, ``count'' means the representation of node $i$ in the last layer can express the number of cycles involving node $i$.
\end{theorem}
Therefore, with enough layers, PST can count the number of cycles of length $\le 7$ between nodes. Given that even 2-FWL is restricted to counting cycles up to length $7$~\citep{2FWL7cycle}, the cycle counting power of our Point Set Transformer is at least on par with 2-FWL.
\section{Related Work}
\textbf{Graph Neural Network with Eigen-Decomposition.}~Our approach employs coordinates derived from the symmetric rank decomposition (SRD) of adjacency or related matrices, differing from prior studies that primarily rely on eigendecomposition (EVD). While both approaches have similarities, SRD transforms the graph isomorphism problem into a set problem \textbf{bijectively}, which is challenging for EVD, because SRD of a matrix is unique up to a single orthogonal transformation, while EVD is unique up to multiple orthogonal transformations in different eigenspaces. This key theoretical difference has profound implications for model design. Early efforts, like \citet{benchmarkingGNN}, introduce eigenvectors into MPNNs' input node feature~\citep{MPNN}, and subsequent works, such as Graph Transformers (GTs)~\citep{firstGT,GT-SAN}, incorporate eigenvectors as node positional encodings. However, due to the non-uniqueness of eigenvectors, these models produce varying predictions for isomorphic graphs, limiting their generalization. \citet{SignBasis-inv} partially solve the non-uniqueness problem. However, their solutions are limited to cases with constant eigenvalue multiplicity in graph tasks due to the property of EVD. On the other hand, approaches like \citet{EquiStableEnc}, \citet{Specformer}, and \citet{StabPE} completely solve non-uniqueness and even apply permutation-equivariant functions to eigenvalues, similar to our PSRD. However, these methods aim to enhance existing MPNNs and GTs with heuristic features. In contrast, we perfectly align graph-level tasks with set-level tasks through SRD, allowing us to convert orthogonal-transformation-equivariant set encoders to graph encoders and to inject graph structure information into Transformers in a principled ways.
\begin{table*}[t]
\caption{Normalized MAE ($\downarrow$) on substructure counting tasks. Following \citet{I2GNN}, models can count the structure if the test loss $\le$ 10 units {({\color{yellow}yellow cell} in the table)}, measured using a scale of $10^{-3}$. TT: Tailed Triangle. CC: Chordal Cycle, TR: Triangle-Rectangle.}\label{tab::subgcount}
\vskip 7pt
\centering
\setlength{\tabcolsep}{1.0pt}
\begin{tabular}{lccccccccccccc}
        \toprule
        Method & 2-Path & 3-Path & 4-Path & 5-path & 6-path & 3-Cycle & 4-Cycle & 5-Cycle & 6-Cycle & 7-cycle &TT & CC & TR \\ \midrule 
        MPNN & \countablecolor 1.0  & 67.3  & 159.2  & 235.3  & 321.5  & 351.5 & 274.2 & 208.8 & 155.5 & 169.8 & 363.1 & 311.4 & 297.9 \\ 
        IDGNN & \countablecolor 1.9  & \countablecolor 1.8  & 27.3  & 68.6  & 78.3  & \countablecolor 0.6 & \countablecolor 2.2 & 49 & 49.5 & 49.9 & 105.3 & 45.4 & 62.8  \\ 
        NGNN & \countablecolor 1.5  & \countablecolor 2.1  & 24.4  & 75.4  & 82.6  & \countablecolor 0.3 & \countablecolor 1.3 & 40.2 & 43.9 & 52.2& 104.4 & 39.2 & 72.9  \\ 
        GNNAK & \countablecolor 4.5  & 40.7  &\countablecolor 7.5  & 47.9  & 48.8  & \countablecolor 0.4 &\countablecolor 4.1 & 13.3 & 23.8 & 79.8 & \countablecolor 4.3 & 11.2 & 131.1  \\ 
        I$^2$-GNN & \countablecolor 1.5  &\countablecolor 2.6  &\countablecolor 4.1  & 54.4  & 63.8  & \countablecolor 0.3 &\countablecolor 1.6 & \countablecolor 2.8 &\countablecolor 8.2 & 39.9 &\countablecolor 1.1 & \countablecolor 1.0 & \countablecolor 1.3  \\ 
        PPGN & \countablecolor 0.3  & \countablecolor 1.7  & \countablecolor 4.1  & 15.1  & 21.7  & \countablecolor 0.3 & \countablecolor 0.9 & \countablecolor 3.6 & \countablecolor 7.1 & 27.1 &\countablecolor 2.6 & \countablecolor 1.5 & 14.4 \\ \midrule
        PSDS & \countablecolor$2.2_{\pm0.1}$ & \countablecolor$2.6_{\pm0.4}$ & \countablecolor \countablecolor $4.9_{\pm 0.8}$ & \countablecolor $9.9_{\pm 0.5}$ & $15.8_{\pm 0.2}$ & \countablecolor $0.6_{\pm 0.7}$ & \countablecolor $2.2_{\pm 0.3}$ & \countablecolor $5.8_{\pm 0.6}$ & $25.1_{\pm 0.7}$ & $57.7_{\pm0.3}$& \countablecolor $6.0_{\pm 1.3}$ & $29.8_{\pm 3.0}$ & $56.4_{\pm 4.7}$ \\
        PST & \countablecolor$0.7_{\pm0.1}$ & \countablecolor$1.1_{\pm0.1}$ & \countablecolor \countablecolor $1.5_{\pm 0.1}$ & \countablecolor $2.2_{\pm 0.1}$ & \countablecolor $3.3_{\pm 0.3}$ & \countablecolor $0.8_{\pm 0.1}$ & \countablecolor $1.9_{\pm 0.2}$ & \countablecolor $3.1_{\pm 0.3}$ & \countablecolor $4.9_{\pm 0.3}$ & \countablecolor $8.6_{\pm0.5}$& \countablecolor $3.0_{\pm 0.1}$ & \countablecolor $4.0_{\pm 0.7}$ & \countablecolor $9.2_{\pm 0.9}$ \\
        \bottomrule
    \end{tabular}
\vskip -15pt
\end{table*}

\textbf{Equivariant Point Cloud and 3-D Molecule Neural Networks.}~Equivariant point cloud and 3-D molecule tasks share resemblances: both involve unordered sets of 3-D coordinate points as input and require models to produce predictions invariant/equivariant to orthogonal transformations and translations of coordinates. Several works~\citep{SE(3)GroupConvolution,SE(3)GroupConvolution1,SE(3)GroupConvolution2,GemNet} introduce specialized equivariant convolution operators to preserve prediction symmetry, yet are later surpassed by models that learn both invariant and equivariant representations for each point, transmitting these representations between nodes. Notably, certain models~\citep{EGNN,PaiNN,VectorNeuron,GNN-LF} directly utilize vectors mirroring input coordinate changes as equivariant features, while others~\citep{TFN,NequIP,SE(3)-Transformers,LieTransformer,HarmonicNet,3DsteerableGNN} incorporate high-order irreducible representations of the orthogonal group, achieving proven universal expressivity~\citep{TFNuniversality}. Our Point Set Transformer (PST) similarly learns both invariant and equivariant point representations. However, due to the specific conversion of point sets from graphs, PST's architecture varies from existing models. While translation invariance characterizes point clouds and molecules, graph properties are sensitive to coordinate translations in our method. Hence, we adopt inner products of coordinates. Additionally, these prior works center on 3D point spaces, whereas our coordinates exist in high-dimensional space, rendering existing models and theoretical expressivity results based on high-order irreducible representations incompatible with our framework.
\section{Experiments}
In our experiments, we evaluate our model across three dimensions: substructure counting for short-range expressivity, real-world graph property prediction for practical performance, and Long-Range Graph Benchmarks~\citep{LRGB} to assess long-range interactions. Our primary model, Point Set Transformer (PST) with PSRD coordinates derived from the Laplacian matrix, performs well on all tasks. Moreover, our graph-to-set method is adaptable to various configurations. In ablation study (see Appendix~\ref{app::abl}), another set encoders Point Set DeepSet (PSDS, introduced in Appendix~\ref{app::PSDS}), SRD coordinates different from PSRD, and coordinates decomposed from the adjacency matrix and normalized adjacency matrix all demonstrate good performance, highlighting the versatility of our approach. Although PST has higher time complexity compared to existing Graph Transformers and is slower on large graphs, it shows similar scalability to our baselines in real-world graph property prediction datasets (see Appendix~\ref{app::scalability}). Our PST uses fewer or comparable parameters than baselines across all datasets. Dataset details, experiment settings, and hyperparameters are provided in Appendix~\ref{app::data} and \ref{app::exp}.
\subsection{Graph substructure counting}
As \citet{CountSubg} highlight, the ability to count substructures is a crucial metric for assessing expressivity. We evaluate our model's substructure counting capabilities on synthetic graphs following \citet{I2GNN}. The considered substructures include paths of lengths 2 to 6, cycles of lengths 3 to 7, and other substructures like tailed triangles (TT), chordal cycles (CC), and triangle-rectangle (TR). Our task involves predicting the number of paths originating from each node and the cycles and other substructures in which each node participates. We compare our Point Set Transformer (PST) with expressive GNN models, including ID-GNNs~\citep{IDGNN}, NGNNs~\citep{NGNN}, GNNAK+\citep{GNNAK}, I$^2$-GNN\citep{I2GNN}, and PPGN~\citep{PPGN}. Baseline results are from \citet{I2GNN}, where uncertainties are unknown.

Results are in Table~\ref{tab::subgcount}. Following \citet{I2GNN}, a model can count a substructure if its normalized test Mean Absolute Error (MAE) is below $10^{-2}$ (10 units in the table). Remarkably, our PST counts all listed substructures, which aligns with our Theorem~\ref{thm::count_path} and Theorem~\ref{thm::count_cycle}, while the second-best model, I$^2$-GNN, counts only 10 out of 13 substructures. PSDS can also count 8 out of 13 substructures, showcasing the versatility of our graph-to-set method.
\begin{table*}[t]
    \centering
    \caption{MAE ($\downarrow$) on the QM9 dataset. * denotes models with 3D coordinates or features as input. LRP: Deep LRP~\citep{CountSubg}. DF: 2-DRFWL(2) GNN~\citep{DRFWL}. 1GNN: 1-GNN. 123: 1-2-3-GNN~\citep{kWL}.}\label{tab::qm9}
\vskip 5pt
    \begin{small}
    \setlength{\tabcolsep}{1pt}
    \begin{tabular}{lccccccc|ccccc}
    \toprule
        Target & Unit & LRP & NGNN & I$^2$GNN & DF & PSDS & PST & 1GNN* & DTNN* & 123* & PPGN* & PST*  \\ 
    \midrule
        $\mu$ & $10^{-1}$D & 3.64 & 4.28 & 4.28 & \underline{3.46} & $3.53_{\pm 0.05}$ & $\mathbf{3.19_{\pm 0.04}}$ & 4.93 & 2.44 & 4.76 & \underline{2.31} & $\mathbf{0.23_{\pm 0.01}}$  \\ 
        $\alpha$ & $10^{-1}$$a_0^3$ & 2.98 & 2.90 & 2.30 & 2.22 & \underline{$2.05_{\pm 0.02}$} & $\mathbf{1.89_{\pm 0.04}}$ & 7.80 & 9.50 & \underline{2.70} & 3.82 & $\mathbf{0.78_{\pm 0.05}}$  \\ 
        $\varepsilon_{\text{homo}}$ & $10^{-2}$meV&  6.91 & 7.21 & 7.10 & \underline{6.15} & $6.56_{\pm 0.03}$ & $\mathbf{5.98_{\pm 0.09}}$ & 8.73 & 10.56 & 9.17 & \underline{7.51} &$\mathbf{2.98_{\pm 0.08}}$ \\ 
        $\varepsilon_{\text{lumo}}$ & $10^{-2}$meV & 7.54 & 8.08 & 7.27 & \underline{6.12} & $6.31_{\pm 0.05}$ & $\mathbf{5.84_{\pm 0.08}}$ & 9.66 & 13.93 & 9.55 & \underline{7.81} & $\mathbf{2.20_{\pm 0.07}}$  \\ 
        $\Delta\varepsilon$ & $10^{-2}$meV & 9.61 & 10.34 & 10.34 & \underline{8.82} & $9.13_{\pm 0.04}$ & $\mathbf{8.46_{\pm 0.07}}$ & 13.33 & 30.48 & 13.06 & \underline{11.05} & $\mathbf{4.47_{\pm 0.09}}$  \\ 
        $R^2$ & $a_0^2$ & 19.30 & 20.50 & 18.64 & 15.04 & \underline{$14.35_{\pm 0.02}$} & $\mathbf{13.08_{\pm 0.16}}$ & 34.10 & 17.00 & 22.90 & \underline{16.07} & $\mathbf{0.93_{\pm 0.03}}$  \\ 
        ZPVE & $10^{-2}$meV & 1.50 & 0.54 & \underline{0.38} & 0.46 & $0.41_{\pm 0.02}$ & $\mathbf{0.39_{\pm 0.01}}$ & 3.37 & 4.68 & \underline{0.52} & 17.42 & $\mathbf{0.26_{\pm 0.01}}$  \\ 
        $U_0$ & meV & 11.24 & 8.03 & 5.74 & 4.24 & \underline{$3.53_{\pm 0.11}$} & $\mathbf{3.46_{\pm 0.17}}$ & 63.13 & 66.12 & {\bf 1.16} & 6.37 & \underline{$3.33_{\pm 0.19}$}  \\ 
        $U$ & meV & 11.24 & 9.82 & 5.61 & 4.16 & $\mathbf{3.49_{\pm 0.05}}$ & \underline{$3.55_{\pm 0.10}$} & 56.60 & 66.12 & {\bf 3.02} & 6.37 & \underline{$3.26_{\pm 0.05}$}  \\ 
        $H$ & meV & 11.24 & 8.30 & 7.32 & 3.95 & $\mathbf{3.47_{\pm 0.04}}$ & \underline{$3.49_{\pm 0.20}$} & 60.68 & 66.12 & {\bf 1.14} & 6.23 & \underline{$3.29_{\pm 0.21}$}  \\ 
        $G$ & meV & 11.24 & 13.31 & 7.10 & 4.24 & \underline{$3.56_{\pm 0.14}$} & $\mathbf{3.55_{\pm 0.17}}$ & 52.79 & 66.12 & {\bf 1.28} & 6.48 & \underline{$3.25_{\pm 0.15}$}  \\ 
        $C_v$ & $10^{-2}$cal/mol/K & 12.90 & 17.40 & {\bf {7.30}} & 9.01 & $8.35_{\pm 0.09}$ & \underline{$7.77_{\pm 0.15}$} & 27.00 & 243.00 & \underline{9.44} & 18.40 & $\mathbf{3.63_{\pm 0.13}}$  \\ 
    \bottomrule
    \end{tabular}
    \end{small}
\vskip -10pt
\end{table*}

\subsection{Graph properties prediction}

We conduct experiments on four real-world graph datasets: QM9~\citep{QM9}, ZINC, ZINC-full~\citep{ZINC}, and ogbg-molhiv~\citep{OGB}. PST excels in performance, and PSDS performs comparable to GIN~\citep{GIN}. PST also outperforms all baselines on TU datasets~\citep{TU} (see Appendix~\ref{app::exp_TU}).

For the QM9 dataset, we compare PST with various expressive GNNs, including models considering Euclidean distances (1-GNN, 1-2-3-GNN~\citep{kWL}, DTNN~\citep{QM9}, PPGN~\citep{PPGN}) and those focusing solely on graph structure (Deep LRP~\citep{CountSubg}, NGNN~\citep{NGNN}, I$^2$-GNN~\citep{I2GNN}, 2-DRFWL(2) GNN~\citep{DRFWL}). For fair comparsion, we introduce two versions of our model: PST without Euclidean distance (PST) and PST with Euclidean distance (PST*). Results in Table~\ref{tab::qm9} show PST outperforms all baseline models without Euclidean distance on 11 out of 12 targets, with an average 11\% reduction in loss compared to the strongest baseline, 2-DRFWL(2) GNN. PST* outperforms all Euclidean distance-based baselines on 8 out of 12 targets, with an average 4\% reduction in loss compared to the strongest baseline, 1-2-3-GNN. Both models rank second in performance for the remaining targets. PSDS without Euclidean distance also outperforms baselines on 6 out of 12 targets.
\begin{table}[t]
\vskip -7pt
    \caption{Results on graph property prediction tasks.}\label{tab::zinc}
\vskip 7pt
    \centering
    \setlength{\tabcolsep}{1pt}
    \begin{small}
    \begin{tabular}{lccc}
    \toprule
        ~ & zinc & zinc-full & molhiv \\
     ~ & MAE$\downarrow$ & MAE$\downarrow$ & AUC$\uparrow$ \\
      \midrule
        GIN & $0.163_{\pm0.004}$ & $0.088_{\pm0.002}$ & $77.07_{\pm1.49}$ \\ 
        GNN-AK+ & $0.080_{\pm0.001}$ & – & $79.61_{\pm1.19}$ \\ 
        ESAN & $0.102_{\pm0.003}$ & $0.029_{\pm0.003}$ & $78.25_{\pm0.98}$ \\ 
        SUN & $0.083_{\pm0.003}$ & $0.024_{\pm0.003}$ & $80.03_{\pm0.55}$ \\ 
        SSWL & $0.083_{\pm0.003}$ & $\underline{0.022_{\pm0.002}}$ & $79.58_{\pm0.35}$ \\ 
        DRFWL & $0.077_{\pm0.002}$ & $0.025_{\pm0.003}$ & $78.18_{\pm2.19}$ \\ 
        CIN & $0.079_{\pm0.006}$ & $\underline{0.022_{\pm0.002}}$ & $\mathbf{80.94_{\pm0.57}}$ \\ 
        NGNN & $0.111_{\pm0.003}$ & $0.029_{\pm0.001}$ & $78.34_{\pm1.86}$ \\ 
        Graphormer & $0.122_{\pm0.006}$ & $0.052_{\pm0.005}$ & $\underline{80.51_{\pm0.53}}$ \\ 
        GPS & $0.070_{\pm0.004}$ & - & $78.80_{\pm1.01}$ \\ 
        GMLP-Mixer &  $0.077_{\pm0.003}$ & - & $79.97_{\pm1.02}$ \\ 
        SAN & $0.139_{\pm0.006}$ & - & $77.75_{\pm0.61}$ \\ 
        Specformer & $0.066_{\pm 0.003}$ & - & $78.89_{\pm 1.24}$\\
        {SignNet} & $0.084_{\pm 0.006}$ & $0.024_{\pm 0.003}$ & -\\
        {Grit} & $\mathbf{0.059_{\pm 0.002}}$ & $0.024_{\pm 0.003}$ & -\\
        \midrule
        PSDS & $0.162_{\pm 0.007}$ & $0.049_{\pm 0.002}$ & $74.92_{\pm 1.18}$\\
        PST & $\underline{0.063_{\pm 0.003}}$ & $\mathbf{0.018_{\pm0.001}}$ & $80.32_{\pm 0.71}$ \\ 
    \bottomrule
    \end{tabular}
    \end{small}
\vskip -10pt
\end{table}

For ZINC, ZINC-full, and ogbg-molhiv datasets, we have conducted an evaluation of PST and PSDS in comparison to a range of expressive GNNs and graph transformers (GTs). This set of models includes expressive MPNN and subgraph GNNs: GIN~\citep{HowPowerfulAreGNNs}, SUN~\citep{SUN}, SSWL~\citep{SSWL}, 2-DRFWL(2) GNN~\citep{DRFWL}, CIN~\citep{CIN}, NGNN~\citep{NGNN}, and GTs:  Graphormer~\citep{GT-graphormer}, GPS~\citep{GT-GPS}, Graph MLP-Mixer~\citep{GraphVit}, Specformer~\citep{Specformer}, {SignNet~\citep{SignBasis-inv}, and Grit~\citep{Grit}}. 
Performance results for the expressive GNNs are sourced from~\citep{DRFWL}, while those for the Graph Transformers are extracted from~\citep{GraphVit,Grit,SignBasis-inv}. The comprehensive results are presented in Table~\ref{tab::zinc}. Notably, our PST outperforms all baseline models on ZINC-full datasets, achieving reductions in loss of 18\%. On the ogbg-molhiv dataset, our PST also delivers competitive results, with only CIN and Graphormer surpassing it. Overall, PST demonstrates exceptional performance across these four diverse datasets, and PSDS also performs comparable to representative GNNs like GIN~\citep{GIN}.

\subsection{Long Range Graph Benchmark}
To assess the long-range capacity of our Point Set Transformer (PST), we conducted experiments using the Long Range Graph Benchmark~\citep{LRGB}. Following \citet{GraphVit}, we compared our model to a range of baseline models, including GCN~\citep{GCN}, GINE~\citep{GIN}, GatedGCN~\citep{GatedGCN}, SAN~\citep{GT-SAN}, Graphormer~\citep{GT-graphormer}, GMLP-Mixer, Graph ViT~\citep{GraphVit}, and Grit~\citep{Grit}. PST outperforms all baselines on the PascalVOC-SP and Peptides-Func datasets and achieves the third-highest performance on the Peptides-Struct dataset. PSDS consistently outperforms GCN and GINE. These results showcase the remarkable long-range interaction capturing abilities of our methods across various benchmark datasets. Note that a contemporary work~\citep{LRGBfail} points out that even vanilla MPNNs can achieve similar performance to Graph Transformers on LRGB with better hyperparameters, which implies that LRGB is not a rigor benchmark. However, for comparison with previous work, we maintain the original settings on LRGB datasets.
\begin{table}[t]
\vskip -7pt
\centering
\caption{Results on Long Range Graph Benchmark. * means using Random Walk Structural Encoding~\citep{RWSE}, and ** means Laplacian Eigenvector Encoding~\citep{benchmarkingGNN}. }
\vskip 5pt
    \setlength{\tabcolsep}{1pt}
\begin{small}
    \begin{tabular}{lccc}
    \toprule
        Model & PascalVOC-SP & Peptides-Func & Peptides-Struct \\ 
        ~ & F1 score $\uparrow$ & AP $\uparrow$ & MAE $\downarrow$ \\
        \midrule
        GCN & $0.1268_{\pm0.0060}$ & $0.5930_{\pm0.0023}$ & $0.3496_{\pm0.0013}$ \\ 
        GINE & $0.1265_{\pm0.0076}$ & $0.5498_{\pm0.0079}$ & $0.3547_{\pm0.0045}$ \\ 
        GatedGCN & $0.2873_{\pm0.0219}$ & $0.5864_{\pm0.0077}$ & $0.3420_{\pm0.0013}$ \\ 
        GatedGCN* & $0.2860_{\pm0.0085}$ & $0.6069_{\pm0.0035}$ & $0.3357_{\pm0.0006}$ \\ 
        Transformer** & $0.2694_{\pm0.0098}$ & $0.6326_{\pm0.0126}$ & $0.2529_{\pm0.0016}$ \\ 
        SAN* & $0.3216_{\pm0.0027}$ & $0.6439_{\pm0.0075}$ & $0.2545_{\pm0.0012}$ \\ 
        SAN** & $0.3230_{\pm0.0039}$ & $0.6384_{\pm0.0121}$ & $0.2683_{\pm0.0043}$ \\ 
        GraphGPS & $0.3748_{\pm0.0109}$ & $0.6535_{\pm0.0041}$ & $0.2500_{\pm0.0005}$ \\ 
        Exphormer & $0.3975_{\pm0.0037}$ & $0.6527_{\pm0.0043}$ & $0.2481_{\pm0.0007}$ \\ 
        GMLP-Mixer & - & $0.6970_{\pm0.0080}$ & $0.2475_{\pm0.0015}$ \\ 
        Graph ViT & - & $0.6942_{\pm0.0075}$ & $\mathbf{0.2449_{\pm0.0016}}$ \\
        Grit & - & $\mathbf{0.6988_{\pm 0.0082}}$ & $0.2460_{\pm 0.0012}$\\
        \midrule
        PSDS & $0.2134_{\pm0.0050}$ & $0.5965_{\pm 0.0064}$ & $0.2621_{\pm 0.0036}$\\
        PST & $\mathbf{0.4010_{\pm0.0072}}$ &  $\mathbf{0.6984_{\pm0.0051}}$ & $0.2470_{\pm 0.0015}$\\
        \midrule
    \end{tabular}
\end{small}
\vskip -15pt
\end{table}
\section{Conclusion}
We introduce a novel approach employing symmetric rank decomposition to transform interconnected nodes in graph into independent points with coordinates. Additionally, we propose the Point Set Transformer to encode the point set. Our approach demonstrates remarkable theoretical expressivity and excels in real-world performance, addressing both short-range and long-range tasks effectively. It extends the design space of GNN and provides a principled way to inject graph structural information into Transformers.

\section{Limitations}
PST's scalability is still constrained by the Transformer architecture. To overcome this, acceleration techniques such as sparse attention and linear attention could be explored, which will be our future work.

\section*{Impact Statement}

This paper presents work whose goal is to advance the field of graph representation learning and will improve the design of graph generation and prediction models. There are many potential societal consequences of our work, none which we feel must be specifically highlighted here.
\section*{Acknowledgement}
Xiyuan Wang and Muhan Zhang are partially supported by the National Key R\&D Program of China (2022ZD0160300), the National Key R\&D Program of China (2021ZD0114702), the National Natural Science Foundation of China (62276003), and Alibaba Innovative Research Program. Pan Li is supported by the National Science Foundation award IIS-2239565.
\bibliography{example_paper}
\bibliographystyle{icml2024}

\newpage
\appendix
\onecolumn

\section{Proof}\label{app::proof}

\subsection{Proof of Proposition~\ref{prop::SRD}}\label{app::proof::srd}
The matrices $Q_1^T Q_1$ and $Q_2^T Q_2$ in $\mathbb{R}^{r \times r}$ are full rank and thus invertible. This allows us to derive the following equations:
\begin{equation}
L = Q_1 Q_1^T = Q_2 Q_2^T
\end{equation}
\begin{align}
Q_1 Q_1^T = Q_2 Q_2^T \Rightarrow Q_1^T Q_1 Q_1^T = Q_1^T Q_2 Q_2^T \\\Rightarrow Q_1^T = (Q_1^T Q_1)^{-1} Q_1^T Q_2 Q_2^T \\
\Rightarrow \exists R \in \mathbb{R}^{m \times m}, Q_1^T = R Q_2^T\\ \Rightarrow \exists R \in \mathbb{R}^{m \times m}, Q_1 = Q_2 R
\end{align}
\begin{align}
Q_1 Q_1^T = Q_2 R R^T Q_2^T = Q_2 Q_2^T \Rightarrow Q_2^T Q_2 R R^T Q_2^T Q_2 = Q_2^T Q_2 Q_2^T Q_2 \\
\Rightarrow RR^T = (Q_2^T Q_2)^{-1} Q_2^T Q_2 Q_2^T Q_2 (Q_2^T Q_2)^{-1} = I
\end{align}

Since $R$ is orthogonal, any two full rank $Q$ matrices are connected by an orthogonal transformation. Furthermore, if there exists an orthogonal matrix $R$ where $R R^T = I$, then $Q_1 = Q_2 R$, and $L = Q_1 Q_1^T = Q_2 R R^T Q_2^T = Q_2 Q_2^T$.

\subsection{Matrix D+A is Always Positive Semi-Definite}\label{app::proof::d+a}
$\forall x\in \sR^n$,
\begin{align}
    x^T(D+A)x&=\sum_{(i,j)\in E} x_ix_j + \sum_{i\in V}(\sum_{j\in V}A_{ij})x_i^2 \\
    &=\sum_{(i,j)\in E} x_ix_j + \frac{1}{2}\sum_{(i,j)\in E}x_i^2+\frac{1}{2}\sum_{(i,j)\in E}x_j^2\\
    &=\frac{1}{2}\sum_{(i,j)\in E} (x_i+x_j)^2\ge 0
\end{align}
Therefore, $D+A$ is always positive semi-definite.

\subsection{Proof of Theorem~\ref{thm::g2s}}
We restate the theorem here:
\begin{theorem}
Given two graphs $\mathcal{G}=(V, A,X)$ and $\mathcal{G'}=(V',A',X')$ with degree matrices $D$ and $D'$, respectively, the two graphs are isomorphic ($\mathcal{G}\simeq \mathcal{G'}$) if and only if $\exists R\in O(r), \msl (X_{v}, RQ_{v})|\forall v \in V\msr = \msl (X_{v}, Q_{v}')|v\in V'\msr$, where $r$ denotes the rank of matrix $Q$, and $Q$ and $Q'$ are the symmetric rank decompositions of $D+A$ and $D'+A'$ respectively.
\end{theorem}
\begin{proof}
Two graphs are isomorphic $\Leftrightarrow$ $\exists \pi \in \Pi_n$, $\pi(A)=A'$ and $\pi(X)=X'$. 

Now we prove that $\exists \pi \in \Pi_n$, $\pi(A)=A'$ and $\pi(X)=X$ $\Leftrightarrow$ $\exists R\in O(r), \msl (X_{v}, RQ_{v})|v \in V\msr = \msl (X_{v}, Q_{v}')|v\in V'\msr$.

When  $\exists \pi \in \Pi_n$, $\pi(A)=A'$ and $\pi(X)=X'$, as 
\begin{equation}
    \pi(Q)\pi(Q)^T = \pi(A+D)=A'+D'=Q'{Q'}^T,
\end{equation}
according to Proposition~\ref{prop::SRD}, $\exists R\in O(r), \pi(Q)R^T=Q'$. Moreover, $\pi(X)=X'$, so
\begin{equation}
    \msl (X_{v}, RQ_{v})|v \in V\msr = \msl (X_{v}', Q_{v}')|v\in V'\msr
\end{equation}

When  $\exists R\in O(r), \msl (X_{v}, RQ_{v})|v \in V\msr = \msl (X_{v}', Q_{v}')|v\in V'\msr$, there exists permutation $\pi\in \Pi_n$, $\pi(X)=X', \pi(Q)R^T=Q'$. Therefore, 
\begin{equation}
    \pi(A+D)=\pi(Q)\pi(Q)^T= \pi(Q)R^TR\pi(Q)^T=Q'{Q'}^T=A'+D'
\end{equation}
\end{proof}

As $A=D+A-\frac{1}{2}\text{diag}((D+A)\vec 1)$, $A'=D+A-\frac{1}{2}\text{diag}((D+A)\vec 1)$, where $\vec 1\in \sR^n$ is an vector with all elements $=1$. 
\begin{equation}
    \pi(A)=A'
\end{equation}

\subsection{Proof of Theorem~\ref{thm::learng2s}}
Now we restate the theorem.
\begin{theorem}
    Given two graphs $\gG=(V, A,X)$ and $\gG=(V', A',X')$, and an injective permutation-equivariant function $Z$ mapping adjacency matrix to a symmetric matrix: (1) For all permutation-equivariant function $f$, if $\gG\simeq \gG'$, then the two sets of PSRD coordinates are equal up to an orthogonal transformation, i.e., $\exists R\in O(r), \msl X_{v}, R\mathcal{Q}(\mathcal{Z}(A), f)_{v}|v\in V\msr= \msl X_{v}', \mathcal{Q}(\mathcal{Z}(A'), f)_{v}|v\in V'\msr$, where $r$ is the rank of $A$, $\mathcal{Q}, \mathcal{Q}'$ are the PSRD coordinates of $A$ and $A'$ respectively. (2) There exists a continuous permutation-equivariant function $f: \sR^r\to \sR^{r\times 2}$, such that $\gG\simeq \gG'$ if $\exists R\in O(r), \forall i=1,2,...,d, \msl (X_{v}, R\mathcal{Q}(\mathcal{Z}(A), f_1)_v, R\mathcal{Q}(\mathcal{Z}(A), f_2)_v)|v\in V\msr= \msl (X_{v}', \mathcal{Q}(\mathcal{Z}(A'), f_1)_v, \mathcal{Q}(\mathcal{Z}(A'), f_2)_v)|v\in V'\msr$, where $f_1: \sR^r\to \sR^r$ and $f_2: \sR^r\to \sR^r$ are two output channels of $f$.
\end{theorem}

\begin{proof}
First, as $\mathcal{Z}$ is an injective permutation equivariant function, forall permutation $\pi\in \Pi_n$
\begin{equation}
    \pi(\mathcal{Z}(A))=\mathcal{Z}(A')\Leftrightarrow \mathcal{Z}(\pi(A))=\mathcal{Z}(A')\Leftrightarrow \pi(A)=A'.
\end{equation}
Therefore, two matrix are isomorphic $\Leftrightarrow $ $\exists \pi\in \Pi_n, \pi(X)=X', \pi(Z)=Z'$, where $Z, Z'$ denote $Z(A), Z(A')$ respectively. 

In this proof, we denote eigendecomposition as $Z=U\text{diag}(\Lambda)U^T$ and $Z'=U'\text{diag}(\Lambda'){U'}^T$, where elements in $\Lambda$ and $\Lambda'$ are sorted in ascending order. Let the multiplicity of eigenvalues in $Z$ be $r_1, r_2,...,r_l$, corresponding to eigenvalues $\lambda_1, \lambda_2,...,\lambda_i$.

(1) If $\gG\simeq \gG'$, there exists a permutation $\pi\in \Pi_n$, $\pi(X)=X', \pi(Z)=Z'$. 
\begin{equation}
    \pi(Z)=Z'\Rightarrow Z'=\pi(U)\text{diag}(\Lambda)\pi(U)^T=U'\text{diag}(\Lambda'){U'}^T.
\end{equation}
$\pi(U)\text{diag}(\Lambda)\pi(U)^T$ is also an eigendecomposition of $Z'$, so $\Lambda=\Lambda'$ as they are both sorted in ascending order. Moreover, since $\pi(U), U'$ are both matrices of eigenvectors, they can differ only in the choice of bases in each eigensubspace. So there exists a block diagonal matrix $V$ with orthogonal matrix $V_1\in O(r_1), V_2\in O(r_2),..., V_l\in O(r_l)$ as diagonal blocks that $\pi(U)V=U'$.

As $f$ is a permutation equivariant function, 
\begin{align}
    \Lambda_i=\Lambda_j&\Rightarrow \exists \pi\in \Pi_r, \pi(i)=j, \pi(j=i), \pi(\Lambda)=\Lambda\\ 
    &\Rightarrow  \exists \pi\in \Pi_r, \pi(i)=j, \pi(j=i), \pi(f(\Lambda))=f(\pi(\Lambda))=f(\Lambda)\\
    &\Rightarrow f(\Lambda)_i=f(\Lambda)_j
\end{align}
Therefore, $f$ will produce the same value on positions with the same eigenvalue. Therefore, $f$ can be consider as a block diagonal matrix with $f_1 I_{r_1}, f_2I_{r_2},...,f_l I_{r_l}$ as diagonal blocks, where $f_i\in \sR$ is $f(\Lambda)_{p_i}$, $p_i$ is a position that $\Lambda_{p_i}=\lambda_i$, and $I_{r}$ is an identity matrix $\in \sR^{r\times r}$. 

Therefore, 
\begin{equation}
    V\text{diag}(f(\Lambda)) = \text{diag}(f_1V_1, f_2V_2,..., f_lV_l)=\text{diag}(f(\Lambda))V.
\end{equation}
Therefore,
\begin{align}
\pi(Q(Z, f))V&=\pi(U)\text{diag}(f(\Lambda))V\\
&=\pi(U)V\text{diag}(f(\Lambda))\\
&=U'\text{diag}(f(\Lambda'))\\
&=Q(Z', f)
\end{align}

As $VV^T=I, V\in O(r)$, 
\begin{equation}
\exists R\in O(r), \msl X_{v}, R\mathcal{Q}(\mathcal{Z}(A), f)_{v}|v\in V\msr= \msl X_{v}', \mathcal{Q}(\mathcal{Z}(A'), f)_{v}|v\in V'\msr
\end{equation}
(2)
We simply define $f_1$ is element-wise abstract value and square root $\sqrt{|.|}$, $f_1$ is element-wise abstract value and square root multiplied with its sign $sgn(|.|)\sqrt{|.|}$. Therefore, $f_1,f_2$ are continuous and permutation equivariant.

if $\exists R\in O(r)$, 

$\msl X_{v}, R\mathcal{Q}(Z(A), f_1)_v, R\mathcal{Q}(Z(A), f_2)_v|v\in V\msr= \msl X_{v}', \mathcal{Q}(Z(A'), f_1)_v, \mathcal{Q}(Z(A'), f_2)_v|v\in V'\msr$. then there exist $\pi\in \Pi_n$, so that
\begin{equation}
    \pi(X)=X'
\end{equation}
\begin{equation}
    \pi(U)\text{diag}(f_1(\Lambda))R^T=U'\text{diag}(f_1(\Lambda'))
\end{equation}
\begin{equation}
    \pi(U)\text{diag}(f_2(\Lambda))R^T=U'\text{diag}(f_2(\Lambda')).
\end{equation}
Therefore,
\begin{align}
    \pi(Z)
    &=\pi(U)\text{diag}(f_1(\Lambda))\text{diag}(f_2(\Lambda))\pi(U')^T\\
    &=\pi(U)\text{diag}(f_1(\Lambda))RR^T\text{diag}(f_2(\Lambda))\pi(U')^T\\
    &=U'\text{diag}(f_1(\Lambda'))\text{diag}(f_2(\Lambda')){U'}^T\\
    &=Z'.
\end{align}
As $\pi(Z)=Z', \pi(X)=X'$, two graphs are isomorphic.

\end{proof}

\subsection{Proof of Theorem~\ref{thm::longrange_mpnn_gt}}\label{app::proof::longrange_mpnn_gt}

Let $H_l$ denote a circle of $l$ nodes. Let $G_l$ denote a graph of two connected components, one is $H_{\lfloor l/2\rfloor}$ and the other is $H_{\lceil l/2 \rceil}$. Obviously, there exists a node pair in $G_l$ with shortest path distance equals to infinity, while $H_l$ does not have such a node pair. So the multiset of shortest path distance is easy to distinguish them. However, they are regular graphs with node degree all equals to $2$, so MPNN cannot distinguish them:

\begin{lemma}
    For all nodes $v, u$ in $G_l$, $H_l$, they have the same representation produced by $k$-layer MPNN, forall $k\in N$.
\end{lemma}
\begin{proof}
    We proof it by induction. 
    \item $k=0$. Initialization, all node with trivial node feature and are the same.
    \item Assume $k-1$-layer MPNN still produce representation $h$ for all node. At the $k$-th layer, each node's representation will be updated with its own representation and two neighbors representations as follows.
    \begin{equation}
        h, \msl h, h\msr
    \end{equation}
    So all nodes still have the same representation.
\end{proof}

\subsection{Proof of Theorem~\ref{thm::2-FWL-spdlimit} and \ref{thm::PST>2-FWL}}

Given two function $f, g$, $f$ can be expressed by $g$ means that there exists a function $\phi$ $\phi\circ g=f$, which is equivalent to given arbitrary input $H, G$, $f(H)=f(G)\Rightarrow g(H)=g(G)$. We use $f\to g$ to denote that $f$ can be expressed with $g$. If both $f\to g$ and $g\to f$, there exists a bijective mapping between the output of $f$ to the output of $g$, denoted as $f\leftrightarrow g$.

Here are some basic rule.
\begin{itemize}
\item $g\to h\Rightarrow f\circ g\to f\circ h$.
\item $g\to h, f\to s\Rightarrow f\circ g\to s\circ h$.
\item $f$ is bijective, $f\circ g\to g$
\end{itemize}

2-folklore Weisfeiler-Leman test produce a color $h_{ij}^t$ for each node pair $(i,j)$ at $t$-th iteration. It updates the color as follows,
\begin{equation}
    h_{ij}^{t+1}=\text{hash}(h_{ij}^{t}, \msl (h_{ik}^t, h_{kj}^t)|k\in V
    \msr).
\end{equation}
The color of the the whole graph is
\begin{equation}
    h_{G}^t = \text{hash}(\msl h_{ij}^t| (i,j)\in V\times V \msr).
\end{equation}

Initially, tuple color hashes the node feature and edge between the node pair, $h_{ij}^0\to \delta_{ij}, A_{ij}, X_i, X_j$.

We are going to prove that
\begin{lemma}
    Forall $t\in \sN$, $h_{ij}^t$ can express $A_{ij}^k$, $k=0,1,2,...,2^t$, where $A$ is the adjacency matrix of the input graph. 
\end{lemma}
\begin{proof}
We prove it by induction on $t$.

\begin{itemize}
    \item When $t=0$, $h_{ij}^0\to A_{ij}, I_{ij}$ in initialization.
    \item If $t>0, \forall t'<t, h_{ij}^{t'}\to A^{k'}, k'=0,1,2,..., 2^{t'}$. For all $k=0,1,2,$
    \begin{align}
    h_{ij}^t
    &\to \text{hash}(h_{ij}^{t}, \msl (h_{ik}^t, h_{kj}^t)|k\in V
    \msr)\\
    &\to \text{hash}(h_{ij}^{t}, \msl (A_{ik}^{\lfloor k/2\rfloor}, A_{kj}^{\lceil k/2\rceil})|k\in V
    \msr)\\
    &\to \sum_{k\in V}A_{ik}^{\lfloor k/2\rfloor}A_{kj}^{\lceil k/2\rceil}\\
    &\to A_{ij}^k
    \end{align}
\end{itemize}
\end{proof}

To prove that $t$-iteration 2-FWL cannot compute shortest path distance larger than $2^{t'}$, we are going to construct an example.

\begin{lemma}
Let $H_l$ denote a circle of $l$ nodes. Let $G_l$ denote a graph of two connected components, one is $H_{\lfloor l/2\rfloor}$ and the other is $H_{\lceil l/2 \rceil}$. 
$\forall K\in \sN^+$, 2-FWL can not distinguish $H_{l_K}$ and $G_{l_K}$, where $l_K=2\times 2\times ( 2^{K})$. However, $G_{l_K}$ contains node tuple with $2^K+1$ shortest path distance between them while $H_{l_K}$ does not, any model count up to $2^K+1$ shortest path distance can count it.  
\end{lemma}

\begin{proof}
Given a fixed $t$, we enumerate the iterations of 2-FWL. Given two graphs $H_{l_K}$, $G_{l_K}$, we partition all tuples in each graph according to the shortest path distance between nodes: $c_0, c_1,..., c_l, ..., c_{2^K}$, where $c_l$ contains all tuples with shortest path distance between them as $l$, and $c_{>2^K}$ contains all tuples with shortest path distance between them larger than $2^K$. We are going to prove that at $k$-th layer $k<=K$, all $c_i, i\le 2^k$ nodes have the same representation (denoted as $h^k_i$) $c_{2^k+1},c_{2^k+2},...,c_{2^K},c_{>K}$ nodes all have the same representation (denoted as $h^k_{2^k+1}$).

Initially, all $c_0$ tuples have representation $h_0^0$, all $c_1$ tuples have the same representation $h_1^0$ in both graph, and all other tuples have the same representation $h_2^0$.

Assume at $k$-th layer, all $c_i, i\le 2^k$ nodes have the same representation $h^k_i$, $c_{2^k+1},c_{2^k+2},...,c_{2^K},c_{>2^K}$ tuples all have the same representation $h^k_{2^k+1}$. At $k+1$-th layer, each representation is updated as follows.
\begin{equation}
    h_{ij}^{t+1}\leftarrow h_{ij}^{t}, \msl (h_{iv}^t, h_{vj}^t)|v\in V
    \msr
\end{equation}
For all tuples, the multiset has $l_K$ elements in total. 

For $c_0$ tuples, the multiset have 1 $(h_0^k, h_0^k)$ as $v=i$, 2 $(h_t^k, h_t^k) $ for $t=1,2,..,2^k$ respectively as v is the $k$-hop neighbor of $i$, and all elements left are $(h_{2^k+1}^k, h_{2^k+1}^k)$ as $v$ is not in the $k$-hop neighbor of $i$. 

For $c_t, t=1,2,...,2^k$ tuples: the multiset have 1 $(h_a^k, h_{t-a}^k)$ for $a=0,1,2,..,t$ respectively as $v$ is on the shortest path between $(i, j)$, and 1 $(h_a^k, h_{2^k+1}^k)$ for $a=1,2,..., 2^k$ respectively, and  1 $( h_{2^k+1}^k, h_a^k)$ for $a=1,2,..., 2^k$ respectively, with other elements are $(h_{2^k+1}^k, h_{2^k+1}^k)$.

For $c_t, t=2^k+1,2^k+2, ..., 2^{k+1}$ tuples: the multiset have 1 $(h_a^k, h_{t-a}^k)$ for $a=t-2^k, t-2^k+1,...,2^k$ respectively as $v$ is on the shortest path between $(i, j)$, 1 $(h_a^k, h_{t-a}^k)$ for $a\in \{0,1,2,...,t-2^k-1\}\cup \{2^k+1, 2^k+2,...,2^{k+1}\}$ respectively as $v$ is on the shortest path between $(i, j)$, and 1 $(h_a^k, h_{2^k+1}^k)$ for $a=1,2,..., 2^k$ respectively, and  1 $( h_{2^k+1}^k, h_a^k)$ for $a=1,2,..., 2^k$ respectively, with other elements are $(h_{2^k+1}^k, h_{2^k+1}^k)$.

For $c_t, t=2^{k+1}+1,...,2^K,>2^K$: the multiset are all the same : 2 $(h_a^k, h_{2^k+1}^k)$ and 2 $(h_{2^k+1}^k, h_a^k)$ for $a=1,2,3,...,2^k$, respectively.
\end{proof}

\subsection{Proof of Theorem~\ref{thm::longrange_pst}}
We can simply choose $f_k(\Lambda)=\Lambda^k$. Then $\langle \mathcal{Q}(A, f_0)_i, \mathcal{Q}(A,f_k)_j\rangle=A_{ij}^k$. The shortest path distance is 
\begin{equation}
    spd(i, j, A)=\arg\min_k \{k\in \sN|A_{ij}^k>0\}
\end{equation}

\subsection{Proof of Theorem~\ref{thm::count_path} and ~\ref{thm::count_cycle}}

This section assumes that the input graph is undirected and unweighted with no self-loops. Let $A$ denote the adjacency matrix of the graph. Note that $A^T=A, A\odot A=A$

An $L$-path is a sequence of edges $[(i_1, i_2),(i_2, i_3), ...,(i_L, i_{L+1})]$, where all nodes are different from each other. An $L$-cycle is an $L$-path except that $i_1 = i_{L+1}$. Two paths/cycles are considered equivalent if their sets of edges are equal. The count of $L$ path from node $u$ to $v$ is the number of non-equivalent pathes with $i_1=u,i_{L+1}=v$. The count of $L$-cycle rooted in node $u$ is the number of non-equivalent cycles involves node $u$.

\citet{PathCount} show that the number of path can be expressive with a polynomial of $A$, where $A$ is the adjacency matrix of the input unweight graph. Specifically, let $P_L$ denote path matrix whose $(u,v)$ elements denote the number of $L$-pathes from $u$ to $v$, \citet{PathCount} provides formula to express $P_L$ with $A$ for small $L$.

This section considers a weaken version of point cloud transformer. Each layer still consists of sv-mixer and multi-head attention. However, the multi-head attention matrix takes the scalar and vector feature before sv-mixer for $Q,K$ and use the feature after sv-mixer for $V$. 

At $k$-th layer
sv-mixer: 
\begin{align}
    s_i'\leftarrow \text{MLP}_1(s_i\Vert \text{diag}(W_1v_i^Tv_i^{k}W_2))\label{app::equ:svmix_s}\\
    v_i'\leftarrow v_i\text{diag}(\text{MLP}_2(s_i'))W_3 +  v_iW_4 \label{app::equ:svmix_v}
\end{align}
Attention layer:
\begin{equation}
    Y_{ij}=\text{MLP}_3(K_{ij}), K_{ij} = (W_q^ss_i\odot W_k^ss_j) \Vert  \text{diag}(W_q^vv_i^Tv_iW_k^v), \label{app::equ:attenmat}
\end{equation}
As $s'_i$ and $v'_i$ can express $s_i, v_i$, so the weaken version can be expressed with the original version.
\begin{align}\label{app::equ::atten}
    s_i&\leftarrow \text{MLP}_4(s_j'\Vert \sum_{j} \text{Atten}_{ij}s_j')\\
    v_i&\leftarrow W_5(v_j'\Vert \sum_{j} \text{Atten}_{ij}v_j')
\end{align}

Let $Y^k$ denote the attention matrix at $k$-th layer. $Y^k$ is a learnable function of $A$. Let $\sY^{k}$ denote the polynomial space of $A$ that $Y^k$ can express. Each element in it is a function from $\sR^{n\times n}\to \sR^{n\times n}$

We are going to prove some lemmas about $\sY$.

\begin{lemma}
$\sY^{k}\subseteq \sY^{k+1}$
\end{lemma}
\begin{proof}
    As there exists residual connection, scalar and vector representations of layer $k+1$ can always contain those of layer $k$, so attention matrix of layer $k+1$ can always express those of layer $k$. 
\end{proof}

\begin{lemma}
    If $y_1,y_2,...,y_s\in \sY^{k}$, their hadamard product $y_1\odot y_2\odot ...\odot y_s\in \sY^{k}$.
\end{lemma}
\begin{proof}
    As $(y_1\odot y_2\odot ...\odot y_s)_{ij}=\prod_{l=1}^{s}(y_l)_{ij}$ is a element-wise polynomial on compact domain, an MLP (denoted as $g$) exists that takes $(i,j)$ elements of the $y_1,y_2,...,y_s$ to produce the corresponding elements of their hadamard product. Assume $g_0$ is the $\text{MLP}_3$ in Equation~\ref{app::equ::atten} to produce the concatenation of $y_1,y_2,..,y_s$, use $g\circ g_0$ (the composition of two mlps) for the $\text{MLP}_3$ in Equation~\ref{app::equ::atten} produces the hadamard product.
\end{proof}

\begin{lemma}
    If $y_1,y_2,...,y_s\in \sY^{k}$, their linear combination $\sum_{l=1}^sa_ly_l \in \sY^{k}$, where $a_l\in \sR$.
\end{lemma}
\begin{proof}
    As $(\sum_{l=1}^sa_ly_l)_{ij}=\sum_{l=1}^sa_l(y_l)_{ij}$ is a element-wise linear layer (denoted as $g$). Assume $g_0$ is the $\text{MLP}_3$ in Equation~\ref{app::equ::atten} to produce the concatenation of $y_1,y_2,..,y_s$, use $g\circ g_0$ for the $\text{MLP}_3$ in Equation~\ref{app::equ::atten} produces the linear combination.
\end{proof}

\begin{lemma}
    $\forall s>0, A^s\in \sY^{1}$.
\end{lemma}
\begin{proof}
    As shown in Section~\ref{sec::g2p_expressivity}, the inner product of coordinates can produce $A^s$.
\end{proof}

\begin{lemma}
    $\forall y_1,y_2,y_3\in \sY^k, s\in \sN^+, d(y_1)y_2,y_2d(y_1),d(y_1)y_2d(y_3),y_1A^s,A^sy_1, y_1A^sy_2\in \sY^k$
\end{lemma}
\begin{proof}
According to Equation~\ref{app::equ:svmix_s} and Equation~\ref{app::equ:attenmat}, $s_i'$ at $k$-th layer can express $y_{ii}$ for all $y\in \sY^k$. Therefore, at $k+1$-th layer in Equation~\ref{app::equ:attenmat}, $\text{MLP}_3$ can first compute element $(i,j)$ $(y_2)_{ij}$ from $s_i,s_j,v_i,v_j$, then multiply $(y_2)_{ij}$ with $(y_1)_{ii}$ from $s_i$, $(y_3)_{jj}$ from $s_j$ and thus produce $d(y_1)y_2,y_2d(y_1),d(y_1)y_2d(y_3)$.

Moreover, according to Equation~\ref{app::equ::atten}, $v_i$ at $k+1$-th layer can express $\sum_k (y_1)_{ik}v_k,\sum_k (y_2)_{ik}v_k$. So at $k+1$-th layer,  the $(i,j)$ element can express $\langle \sum_k (y_1)_{ik}v_k, v_j\rangle, \langle v_i, \sum_k (y_1)_{jk}v_k\rangle,  \langle (y_1)_{ik}v_k, \sum_k (y_2)_{jk}v_k\rangle$, corresponds to $y_1A^s, A^sy_2,y_1A^sy_2$, respectively.

\end{proof}

Therefore, 
\begin{lemma}
\begin{itemize}
    \item $\forall s>0, l>0, a_i>0$, $\odot_{i=1}^{l}A^{a_i}\in \sY^1$.
    \item $\forall s_1,s_2>0, l>0$, $A^{s_1}d(\odot_{i=1}^{l}A^{a_i}), d(\odot_{i=1}^{l}A^{a_i})A^{s_1},d(\odot_{i=1}^{l_1}A^{b_i}) A^{s_1}d(\odot_{i=1}^{l_2}A^{b_i})\in \sY^2$.
    \item $\forall s_1,s_2,s_3>0$, $A^{s_1}d(\odot_{i=1}^{l}A^{a_i})$
\end{itemize}
\end{lemma}

Therefore, we come to our main theorem.
\begin{theorem}\label{app::thm::countpath}
    The attention matrix of 1-layer PST can express $P_2$, 2-layer PST can express $P_3$, 3-layer PST can express $P_4,P_5$, 5-layer PST can express $P_6$.  
\end{theorem}

\begin{proof}
As shown in \citep{PathCount}, 
\begin{align}
    P_2&=A^2
\end{align}
Only one kind basis $\odot_{i=1}^{l}A^{a_i}$. 1-layer PST can express it.

\begin{align}
    P_3&=A^3 + A - Ad(A^2)-d(A^2)A
\end{align}
Three kind of basis $\odot_{i=1}^{l}A^{a_i}$($A^3, A$), $A^{s_1}d(\odot_{i=1}^{l}A^{a_i})$($Ad(A^2)$), and $d(\odot_{i=1}^{l}A^{a_i})A^{s_1}$. 2-layer PST can express it.

\begin{align}
    P_4&=A^4 + A^2 + 3A\odot A^2-d(A^3)A-d(A^2)A^2-Ad(A^3)-A^2d(A^2)-Ad(A^2)A
\end{align}
Four kinds of basis $\odot_{i=1}^{l}A^{a_i}$ ($A^4,A^2, A\odot A^2$), $A^{s_1}d(\odot_{i=1}^{l}A^{a_i})$ ($Ad(A^3),A^2d(A^2)$), $d(\odot_{i=1}^{l}A^{a_i})A^{s_1}$ ($d(A^3)A,d(A^2)A^2$), and $A^{s_1}d(\odot_{i=1}^{l}A^{a_i})A^{s_3}$ ($Ad(A^2)A$). 3-layer PST can express it.

\begin{align}
    P_5&=A^5+3A^3+ 4A\\ 
    &~+3 A^2\odot A^2\odot A+3A\odot A^3-4A\odot A^2 \\
    &~-d(A^4)A-d(A^3)A^2-d(A^2)A^3+2d(A^2)^2A-2d(A^2)A- 4d(A^2)A\\
    &~-Ad(A^4)-A^2d(A^3)-A^3d(A^2)+2Ad(A^2)^2-2Ad(A^2)-4Ad(A^2)\\
    &~+d(A^2)Ad(A^2)\\
    &~+3(A\odot A^2)A\\
    &~+3A(A\odot A^2)\\
    &~-Ad(A^3)A-Ad(A^2)A^2-A^2d(A^2)A\\
    &~+d(Ad(A^2)A)A\\
\end{align}
Basis are in 
\begin{itemize}
    \item $\sY^1$
    \begin{itemize}
        \item $\odot_{i=1}^{l}A^{a_i}$: $A^5,A^3, A, A^2\odot A^2\odot A,A\odot A^3,A\odot A^2$.
    \end{itemize}
    \item $\sY^2$ 
    \begin{itemize}
        \item $A^{s_1}d(\odot_{i=1}^{l}A^{a_i})$:$Ad(A^4),A^2d(A^3),A^3d(A^2),Ad(A^2)^2,Ad(A^2),Ad(A^2)$.
        \item $d(\odot_{i=1}^{l}A^{a_i})A^{s_1}$: $Ad(A^4),A^2d(A^3),A^3d(A^2),Ad(A^2)^2,Ad(A^2),Ad(A^2)$.
        \item $d(\odot_{i=1}^{l_1}A^{a_i})A^{s_1}d(\odot_{i=1}^{l_1}A^{b_i})$: $d(A^2)Ad(A^2)$.
        \item $A^{s_1}(\odot_{i=1}^{l}A^{a_i})$: $A(A\odot A^2)$.
        \item $(\odot_{i=1}^{l}A^{a_i})A^{s_1}$: $(A\odot A^2)A$
    \end{itemize}
    \item $\sY^3$: 
    \begin{itemize}
        \item $A^{s}\sY^2$: $Ad(A^3)A$, $Ad(A^2)A^2$, $A^2d(A^2)A$.
        \item $d(\sY^2)A^s$: $d(Ad(A^2)A)A$
    \end{itemize}
\end{itemize}
3-layer PST can express it.

Formula for $6$-path matrix is quite long. 
\begin{align}
    P_6&=A^6+4A^4+12A^2 \\
    &~+3A\odot A^4+6A\odot A^2\odot A^3+A^2\odot A^2\odot A^2-4A^2\odot A^2+44A\odot A^2\\
    &~-d(A^5)A-d(A^4)A^2-d(A^3)A^3-5d(A^3)A-d(A^2)A^4-7d(A^2)A^2\\
    &+2d(A^2)^2A^2+4(d(A^2)\odot d(A^3))A\\
    &~-Ad(A^5)-A^4d(A^2)-A^3d(A^3)-5Ad(A^3)-A^2d(A^4)-7A^2d(A^2)\\
    &+2A^2d(A^2)^2+4A(d(A^2)\odot d(A^3))\\
    &~+d(A^2)Ad(A^3)+d(A^3)Ad(A^2)+ d(A^2)A^2d(A^2)\\
    &~+2(A\odot A^3)A+2(A\odot A^2\odot A^2)A+(A^2\odot A^2\odot A)A-3(A\odot A^2)A+(A\odot A^3)A\\
    &~+(A\odot A^2)A^2+2(A\odot A^2)A^2-(A\odot A^2)A\\ 
    &~+2A(A\odot A^3)+2A(A\odot A^2\odot A^2)+A(A^2\odot A^2\odot A)-3A(A\odot A^2)+A(A\odot A^3)\\
    &+A^2(A\odot A^2)+2A^2(A\odot A^2)-A(A\odot A^2)\\
    &~-8A\odot(A(A\odot A^2))-8A\odot ((A\odot A^2)A)\\
    &~-12 d(A^2)(A\odot A^2)-12(A\odot A^2)d(A^2)\\
    &~-Ad(A^4)A-Ad(A^2)A^3-A^3d(A^2)A-Ad(A^3)A^2-A^2d(A^3)A-A^2d(A^2)A^2\\
    &-10Ad(A^2)A+2Ad(A^2)^2A\\
    &~+d(A^2)Ad(A^2)A+Ad(A^2)Ad(A^2)\\
    &~-3A\odot (Ad(A^2)A)\\
    &~-4Ad((A\odot A^2)A)-4Ad(A(A\odot A^2))\\
    &~+3A(A\odot A^2)A\\
    &~-4d(A(A\odot A^2))A-4d((A\odot A^2)A)A\\
    &~+d(Ad(A^3)A)A+d(Ad(A^2)A)A^2+d(Ad(A^2)A^2)A+d(A^2d(A^2)A)A\\
    &~+Ad(Ad(A^3)A)+A^2d(Ad(A^2)A)+Ad(A^2d(A^2)A)+Ad(Ad(A^2)A^2)\\
    &~+Ad(Ad(A^2)A)A\\
\end{align}

Basis are in 
\begin{itemize}
    \item $\sY^1$
    \begin{itemize}
        \item $\odot_{i=1}^{l}A^{a_i}$: $A^6,A^4,A^2,A\odot A^4,A\odot A^2\odot A^3,A^2\odot A^2\odot A^2,A^2\odot A^2,A\odot A^2$.
    \end{itemize}
    \item $\sY^2$ 
    \begin{itemize}
        \item $A^{s_1}d(\odot_{i=1}^{l}A^{a_i})$:$Ad(A^5)$, $A^4d(A^2)$, $A^3d(A^3)$, $Ad(A^3)$, $A^2d(A^4)$, $A^2d(A^2)$, $A^2d(A^2)^2$, $A(d(A^2)\odot d(A^3))$.
        \item $d(\odot_{i=1}^{l}A^{a_i})A^{s_1}$: $d(A^5)A$, $d(A^4)A^2$, $d(A^3)A^3$, $d(A^3)A$, $d(A^2)A^4$, $d(A^2)A^2$, $d(A^2)^2A^2$, $(d(A^2)\odot d(A^3))A$.
        \item $d(\odot_{i=1}^{l_1}A^{a_i})A^{s_1}d(\odot_{i=1}^{l_1}A^{b_i})$: $d(A^2)Ad(A^3)$, $d(A^3)Ad(A^2)$, $d(A^2)A^2d(A^2)$.
        \item $A^{s_1}(\odot_{i=1}^{l}A^{a_i})$: $A(A\odot A^3)$, $A(A\odot A^2\odot A^2)$, $A(A^2\odot A^2\odot A)$, $A(A\odot A^2)$, $A(A\odot A^3)$, $A^2(A\odot A^2)$, $A^2(A\odot A^2)$, $A(A\odot A^2)$.
        \item $(\odot_{i=1}^{l}A^{a_i})A^{s_1}$: $(A\odot A^3)A$, $(A\odot A^2\odot A^2)A$, $(A^2\odot A^2\odot A)A$, $(A\odot A^2)A$, $(A\odot A^3)A$, $(A\odot A^2)A^2$, $(A\odot A^2)A^2$, $(A\odot A^2)A$
        \item $\sY^2\odot \sY^2$: $A\odot ((A\odot A^2)A)$, $A\odot ((A\odot A^2)A)$.
        \item $d(\sY^1)\sY^1$: $d(A^2)(A\odot A^2)$
        \item $\sY^1d(\sY^1)$: $(A\odot A^2)d(A^2)$
    \end{itemize}
    \item $\sY^3$: 
    \begin{itemize}
        \item $A^{s}\sY^2$: $Ad(A^4)A$, $Ad(A^2)A^3$, $A^3d(A^2)A$, $Ad(A^3)A^2$, $A^2d(A^3)A$, $A^2d(A^2)A^2$, $Ad(A^2)A$, $Ad(A^2)^2A$, $ Ad(A^2)Ad(A^2)$, $A(A\odot A^2)A$.
        \item $\sY^2A^{s}$: $d(A^2)Ad(A^2)A$.
        \item $\sY^3\odot\sY^3$: $A\odot (Ad(A^2)A)$. 
        \item $d(\sY^2)\sY^2$: $d(Ad(A^2)A)A$,$d(A(A\odot A^2))A$,$d((A\odot A^2)A)A$
        \item $\sY^2d(\sY^2)$:$Ad((A\odot A^2)A)$,$Ad(A(A\odot A^2))$
    \end{itemize}
    \item $\sY^4$:
    \begin{itemize}
        \item $d(\sY^3)\sY^3$: $d(Ad(A^3)A)A$, $d(Ad(A^2)A)A^2$, $d(Ad(A^2)A^2)A$, $d(A^2d(A^2)A)A$.
        \item $\sY^3d(\sY^3)$:  $Ad(Ad(A^3)A)$, $A^2d(Ad(A^2)A)$, $Ad(A^2d(A^2)A)$, $Ad(Ad(A^2)A^2)$.
    \end{itemize}
    \item $\sY^5$:
    \begin{itemize}
        \item $A^{s_1}d(\sY^3)A^{s_2}$:$Ad(Ad(A^2)A)A$
    \end{itemize}
\end{itemize}
5-layer PST can express it.
\end{proof}

Count cycle is closely related to counting path. A $L+1$ cycle contains edge $(i,j)$ can be decomposed into a $L$-path from $i$ to $j$ and edge $(i,j)$. Therefore, the vector of count of cycles rooted in each node $C_{L+1}=\text{diagonal}(AP_{L})$
\begin{theorem}
The diagonal elements of attention matrix of 2-layer PST can express $C_3$, 3-layer PST can express $C_4$, 4-layer PST can express $C_5,C_6$, 6-layer PST can express $C_7$.  
\end{theorem}
\begin{proof}
    It is a direct conjecture of Theorem~\ref{app::thm::countpath} as $C_{L+1}=\text{diagonl}(AP_{L})$ and $\forall k, P_{L}\in \sY^k\Rightarrow AP_{L}\in \sY^{k+1}$
\end{proof}

\section{Expressivity Comparision with Other Models}\label{app::expressivity}

Algorithm A is considered more expressive than algorithm B if it can differentiate between all pairs of graphs that algorithm B can distinguish. If there is a pair of links that algorithm A can distinguish while B cannot and A is more expressive than B, we say that A is strictly more expressive than B. We will first demonstrate the greater expressivity of our model by using PST to simulate other models. Subsequently, we will establish the strictness of our model by providing a concrete example.

Our transformer incorporates inner products of coordinates, which naturally allows us to express shortest path distances and various node-to-node distance metrics. These concepts are discussed in more detail in Section~\ref{sec::g2p_expressivity}. This naturally leads to the following theorem, which compares our PST with GIN~\citep{GIN}.

\begin{theorem}\label{thm::PST>GIN}
    A $k$-layer Point Set Transformer is strictly more expressive than a $k$-layer GIN.
\end{theorem}
\begin{proof}
We first prove that one PST layer can simulate an GIN layer.

Given node features $s_i$ and $v_i$. Without loss of generality, we can assume that one channel of $v_i$ contains $U\text{diag}(\Lambda^1/2)$. The sv-mixer can simulate an MLP function applied to $s_i$. Leading to $s_i'$. A GIN layer will then update node representations as follows,

\begin{equation}
    s_i\leftarrow s_i' + \sum_{j\in N(i)}s_j'
\end{equation}

By inner products of coordinates, the attention matrix can express the adjacency matrix. By setting  $W_q^s,W_k^s=0$, and $W_q^v, W_k^v$ be a diagonal matrix with only the diagonal elements at the row corresponding the the channel of $U\text{diag}(\Lambda^1/2)$.
\begin{equation}
    K_{ij} = (W_q^ss_i \odot W_k^ss_j) \Vert \text{diagonal}(W_q^vv_i^Tv_jW_k^v)
    \to \langle \text{diag}(\Lambda^1/2)U_i,  \text{diag}(\Lambda^1/2)U_j\rangle 
    =A_{ij}
\end{equation}
Let $\text{MLP}$ express an identity function. 
\begin{equation}
\text{Atten}_{ij} = \text{MLP}(K_{ij})\to  A_{ij}
\end{equation}
The attention layer will produce
\begin{equation}
    s_i\leftarrow \sum_j A_{ij}s_j' =\sum_{j\in N(i)}s_j'
\end{equation}
with residual connection, the layer can express GIN
\begin{equation}
    s_i\leftarrow s_i' + s_i = s_i' + \sum_{j\in N(i)}s_j'
\end{equation}
Moreover, as shown in Theorem~\ref{thm::longrange_mpnn_gt}, MPNN cannot compute shortest path distance, while PST can. So PST is strictly more expressive.
\end{proof}

Moreover, our transformer is strictly more expressive than some representative graph transformers, including Graphormer~\citep{GT-graphormer} and GPS with RWSE as structural encoding~\citep{GT-GPS}.

\begin{theorem}
A $k$-layer Point Set Transformer is strictly more expressive than a $k$-layer Graphormer and a $k$-layer GPS.
\end{theorem}
\begin{proof}
We first prove that $k$-layer Point Set Transformer is more expressive than a $k$-layer Graphormer and a $k$-layer GPS.

In initialization, besides the original node feature, Graphormer further add node degree features and GPS further utilize RWSE. Our PST can add these features with the first sv-mixer layer.

\begin{equation}
    s_i' \leftarrow \text{MLP}_1(s_i \Vert \text{diagonal}(W_1v_i^Tv_iW_2^T))
\end{equation}
Here, $\text{diagonal}(W_1v_iv_i^TW_2^T)$ add coordinate inner products, which can express RWSE (diagonal elements of random walk matrix) and degree (see Appendix~\ref{app::SE}), to node feature.

Then we are going to prove that one PST layer can express one GPS and one Graphormer layer. PST's attention matrix is as follows, 
\begin{equation}
    \text{Atten}_{ij} = \text{MLP}(K_{ij}), \quad K_{ij} = (W_q^ss_i \odot W_k^ss_j) \Vert \text{diagonal}(W_q^vv_i^Tv_jW_k^v)
    \to \langle \text{diag}(\Lambda^1/2)U_i,  \text{diag}(\Lambda^1/2)U_j\rangle 
\end{equation}
The Hadamard product $(W_q^ss_i \odot W_k^ss_j)$ with MLP can express the inner product of node representations used in Graphormer and GPS. The inner product of coordinates can express adjacency matrix used in GPS and Graphormer and shortest path distance used in Graphormer. Therefore, PST' attention matrix can express the attention matrix in GPS and Graphormer. 

To prove strictness, Figure 2(c) in \citep{biconnectivity} provides an example. As PST can capture resistance distance and simulate 1-WL, so it can differentiate the two graphs according to Theorem 4.2 in \citep{biconnectivity}. However, Graphormer cannot distinguish the two graphs, as proved in \citep{biconnectivity}. 

For GPS, Two graphs in Figure 2(c) have the same RWSE: RWSE is
\begin{equation}
    \text{diagonal}(\hat U\text{diag}(\hat \Lambda^k)\hat U^T), k=1,2,3,...,
\end{equation}
where the eigendecomposition of normalized adjacency matrix $D^{-1/2}AD^{-1/2}$ is $\hat U$. By computation, we find that two graphs share the same $\hat \Lambda$. Moroever, $\text{diagonal}(\hat U\text{diag}(\hat \Lambda^k)\hat U^T)$ are equal in two graphs for $k=0,1,2,..., 9$, where $9$ is the number of nodes in graphs. $\Lambda^k$ and $\text{diagonal}(\hat U\text{diag}(\hat \Lambda^k)\hat U^T)$ with larger $k$ are only  linear combinations of $\Lambda^k$ and thus $\text{diagonal}(\hat U\text{diag}(\hat \Lambda^k)\hat U^T)$ for $k=0, 1, ..., 9$. So the RWSE in the two graphs are equal and equivalent to simply assigning feature $h_1$ to the center node and feature $h_2$ to other nodes in two graphs. Then GPS simply run a model be a submodule of Graphormer on the graph and thus cannot differentiate the two graphs either.
\end{proof}

Even against a highly expressive method such as 2-FWL, our models can surpass it in expressivity with a limited number of layers:

\begin{theorem}\label{thm::PST>2-FWL}
For all $K>0$, a graph exists that a $K$-iteration 2-FWL method fails to distinguish, while a $1$-layer Point Set Transformer can.
\end{theorem} 
\begin{proof}
It is a direct corollary of Theorem~\ref{thm::2-FWL-spdlimit}. 
\end{proof}

\section{Some Graph Transformers Fail to Compute Shortest Path Distance}\label{app::GT_fail_spd}

First, we demonstrate that computing inner products of node representations alone cannot accurately determine the shortest path distance when the node representations are permutation-equivariant. Consider Figure~\ref{fig:GAEfailure} as an illustration. In cases where node representations exhibit permutation-equivariance, nodes $v_2$ and $v_3$ will share identical representations. Consequently, the pairs $(v_1, v_2)$ and $(v_1, v_3)$ will yield the same inner products of node representations, despite having different shortest path distances. Consequently, it becomes evident that the attention matrices of some Graph Transformers are incapable of accurately computing the shortest path distance.

\begin{theorem}
    GPS with RWSE~\citep{GT-GPS} and Graphormer without shortest path distance encoding cannot compute shortest path distance with the elements of adjacency matrix.
\end{theorem}
\begin{proof}
Their adjacency matrix elements are functions of the inner products of node representations and the adjacency matrix.

\begin{equation}
    \text{Atten}_{ij} = \langle s_i, s_j\rangle |\!|A_{ij}.
\end{equation}
This element is equal for the node pair $(v_1, v_2)$ and $(v_1, v_3)$ in Figure~\ref{fig:GAEfailure}. However, two pairs have different shortest path distances.
\end{proof}

Furthermore, while Graph Transformers gather information from the entire graph, they may not have the capacity to emulate multiple MPNNs with just a single transformer layer. To address this, we introduce the concept of a \textit{vanilla Graph Transformer}, which applies a standard Transformer to nodes using the adjacency matrix for relative positional encoding. This leads us to the following theorem.
\begin{theorem}
    For all $k\in \sN$, there exists a pair of graph that $k+1$-layer MPNN can differentiate while $k$-layer MPNN and $k$-layer vanilla Graph Transformer cannot.
\end{theorem}
\begin{proof}
Let $H_l$ denote a circle of $l$ nodes. Let $G_l$ denote a graph of two components, one is $H_{\lfloor l/2\rfloor}$ and the other is $H_{\lceil l/2 \rceil}$. Let $H_l'$ denote adding a unique feature $1$ to a node in $H_l$ (as all nodes are symmetric for even $l$, the selection of node does not matter), and $G_l'$ denote adding a unique feature $1$ to one node in $G_l$. All other nodes have feature $0$. Now we prove that

\begin{lemma}
For all $K\in \sN$, $(K+1)$-layer MPNN can distinguish $H'_{4(K+1)}$ and $G'_{4(K+1)}$, while $K$-layer MPNN and $K$-layer vanilla Graph Transformer cannot distinguish. 
\end{lemma}

Given $H'_{4(K+1)}$, $G'_{4(K+1)}$, we assign each node a color according to its distance to the node with extra label $1$: $c_0$ (the labeled node itself), $c_1$ (two nodes connected to the labeled node), $c_2$ (two nodes whose shortest path distance to the labeled node is $2$),..., $c_K$ (two nodes whose shortest path distance to the labeled node is $K$), $c_{>K}$ (nodes whose shortest path distance to the labeled node is larger than $K$). Now by simulating the process of MPNN, we prove that at $k$-th layer $k<=K$, $\forall i\le k$, $c_i$ nodes have the same representation (denoted as $h^k_i$), respectively, $c_{k+1},c_{k+2},...,c_{K},c_{>K}$ nodes all have the same representation (denoted as $h^k_{k+1}$).

Initially, all $c_0$ nodes have representation $h_0^0$, all other nodes have representation $h_1^0$ in both graph.

Assume at $k$-th layer, $\forall i\le k$, $c_i$ nodes have the same representation $h^k_i$, respectively, $c_{k+1},c_{k+1},...,c_{K},c_{>K}$ nodes all have the same representation $h^k_{k+1}$. At $k+1$-th layer, $c_0$ node have two neighbors of representation $h_1^k$. all $c_i, 1<i<=k$ node two neighbors of representations $h^k_{i-1}$ and $h^k_{i+1}$, respectively. $c_{k+1}$ nodes have two neighbors of representation $h^k_k$ and $h^{k}_{k+1}$. All other nodes have two neighbors of representation $h^k_{k+1}$. So $c_i, i\le k+1$ nodes have the same representation (denoted as $h^{k+1}_i$), respectively, $c_{k+1+1},...,c_{K},c_{>K}$ nodes all have the same representation (denoted as $h^k_{k+1}$).

The same induction also holds for $K$-layer vanilla graph transformer.

However, in the $K+1$-th message passing layer, only one node in $G_{4(K+1)}$ is of shortest path distance $K+1$ to the labeled node. It also have two neighbors of representation $h^{K}_{K}$. While such a node is not exist in $H_{4(K+1)}$. So ($K+1$)-layer MPNN can distinguish them.
\end{proof}

The issue with a vanilla Graph Transformer is that, although it collects information from all nodes in the graph, it can only determine the presence of features in 1-hop neighbors. It lacks the ability to recognize features in higher-order neighbors, such as those in 2-hop or 3-hop neighbors. A straightforward solution to this problem is to manually include the shortest path distance as a feature. However, our analysis highlights that aggregating information from the entire graph is insufficient for capturing long-range interactions. 

\begin{figure}[t]
  \begin{center}
    \includegraphics[width=0.25\textwidth]{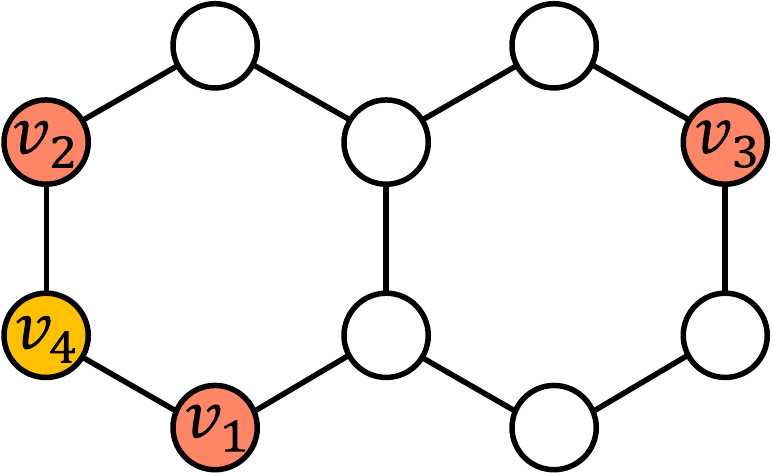} 
  \end{center}
  \vskip -10pt
  \caption{The failure of using inner products of permutation-equivariant node representations to predict shortest path distance. $v_2$ and $v_3$ have equal node representations due to symmetry. Therefore, $(v_1, v_2)$ and $(v_1, v_3)$ will have the same inner products of node representations but different shortest path distance.}\label{fig:GAEfailure}
  \vskip -10pt
\end{figure}

\begin{table}[t]
\vskip -10pt
    \centering
    \caption{Connection between existing structural embeddings and our parameterized coordinates. The eigendecomposition are $\hat A\leftarrow \hat U\hat \Lambda \hat U$, $D-A\leftarrow \tilde U\tilde \Lambda \tilde U^T$, $A\leftarrow U\Lambda U^T$. $d_i$ denote the degree of node $i$.}
    \label{tab:parameter}
\vskip 5pt
    \begin{tabularx}{1.0\textwidth} {
    >{\hsize=.7\hsize\linewidth=\hsize}X
    >{\hsize=1.3\hsize\linewidth=\hsize}X
    >{\hsize=1.0\hsize\linewidth=\hsize}X
   }

             \toprule Method& Description & Connection \\
         \midrule
         Random walk matrix~\citep{DE,benchmarkingGNN,GT-GPS}&  $k$-step random walk matrix is $(D^{-1}A)^k$, whose element $(i,j)$ is the probability that a $k$-step random walk starting from node $i$ ends at node $j$. & \begin{small}$(D^{-1}A)^k_{ij}$
         
        $= (D^{-1/2}(\hat A)^k D^{1/2})_{ij}$
        
        $= \sqrt{\frac{d_j}{d_i}}\langle \hat U_{i}, \text{diag}(\hat \Lambda^k) \hat U_j\rangle$\end{small}
        
        \\

        \midrule
        Heat kernel matrix~\citep{GT-graphit}& Heat kernel is a solution of the heat equation. Its element $(i,j)$ represent how much heat diffuse from node $i$ to node $j$ &\begin{small}     
        $\!(\!I\!+\!\tilde U(\text{diag}(\exp(-t\tilde \Lambda))\!-\!I)\tilde U^T)_{ij}$
        
        $=\!\!\delta_{ij}\!+\!\langle \tilde U_{i}, (\text{diag}(\exp(-t\tilde \Lambda))\!-\!I) \tilde U_{j}\rangle$\end{small} \\ 
        \midrule
    Resistance distance~\citep{NGNN,biconnectivity} & Its element $(i,j)$ is the resistance between node $i,j$ considering the graph as an electrical network. It is also the pseudo-inverse of laplacian matrix $L$, & \begin{small}
    $(\tilde U\text{diag}(\tilde \Lambda^{-1})\tilde U^T)_{ij}$
    
    $=\langle \tilde U_i, \text{diag}(\tilde \Lambda^{-1})\tilde U_j\rangle$\end{small}\\
    \midrule
    Equivariant and stable laplacian PE~\citep{EquiStableEnc}& The encoding of node pair $i,j$ is $\Vert 1_{K}\odot
     (U_{i}-U_{j}) \Vert$, where $1_K$ means a vector $\in \sR^r$ with its elements coresponding to $K$ largest eigenvalue of $L$ & \begin{small}$
     \Vert 1_{K}\odot (U_{i}-U_{j}) \Vert^2$

     $=\langle U_{i}, \text{diag}(1_K)U_{i} \rangle$
     
     $+\langle U_{j}, \text{diag}(1_K)U_{j} \rangle$

     $-2\langle U_{i}, \text{diag}(1_K)U_{j} \rangle$
     \end{small}
     \\
    \midrule
    Degree and number of triangular~\citep{trianglecounting} & $d_i$ is the number of edges, and $t_i$ is the number of triangular rooted in node $i$.
    &\begin{small}$d_i=\langle U_i,\text{diag}(\Lambda^2) U_j\rangle$.

    $t_i=\langle U_i,\text{diag}(\Lambda^3) U_j\rangle$\end{small}\\
    \bottomrule 
    \end{tabularx}
\end{table}
\section{Connection with Structural Embeddings}\label{app::SE}
We show the equivalence between the structural embeddings and the inner products of our PSRD coordinates in Table~\ref{tab:parameter}. The inner products of PSRD coordinates can unify a wide range of positional encodings, including random walk~\citep{DE,benchmarkingGNN,GT-GPS}, heat kernel~\citep{GT-graphit}, and resistance distance~\citep{NGNN,biconnectivity}.

\section{Datasets}\label{app::data}

\begin{table}[t]
    \centering
    \caption{Statistics of the datasets. \#Nodes and \#Edges denote the number of nodes and edges per graph. In split column, 'fixed' means the dataset uses the split provided in the original release. 
 Otherwise, it is of the formal training set ratio/valid ratio/test ratio.}
    \label{tab::data}
    \begin{small}
    \begin{tabular}{ccccccc}
    \toprule
        ~ & \#Graphs &  \#Nodes &  \#Edges & Task & Metric & Split \\ 
    \midrule
        Synthetic & 5,000 & 18.8  & 31.3  & Node Regression & MAE & 0.3/0.2/0.5. \\ 
        QM9 & 130,831 & 18.0  & 18.7  &  Regression & MAE & 0.8/0.1/0.1 \\ 
        ZINC & 12,000 & 23.2  & 24.9  &  Regression & MAE & fixed \\ 
        ZINC-full & 249,456 & 23.2  & 24.9  &  Regression & MAE & fixed \\ 
        ogbg-molhiv & 41,127 & 25.5  & 27.5  &  Binary classification & AUC & fixed \\ 
        MUTAG &188 & 17.9 & 39.6 & classification & Accuracy & 10-fold cross validataion\\
        PTC-MR &344 & 14.3 & 14.7 & classification & Accuracy&10-fold cross validation\\
        PROTEINS &1113 & 39.1 & 145.6 & classification & Accuracy&10-fold cross validataion\\
        IMDB-BINARY & 1000 & 19.8 & 193.1 & classification & Accuracy&10-fold cross validataion\\
        PascalVOC-SP & 11,355 & 479.4  & 2710.5  & Node Classification & Macro F1 & fixed \\ 
        Peptides-func & 15,535 & 150.9  & 307.3  & Classification & AP & fixed \\ 
        Peptides-struct 1 & 15,535 & 150.9  & 307.3  & Regression & MAE & fixed \\ 
    \bottomrule
    \end{tabular}
    \end{small}
\end{table}

We summarize the statistics of our datasets in Table~\ref{tab::data}. Synthetic is the dataset used in substructure counting tasks provided by ~\citet{I2GNN}, they are random graph with the count of substructure as node label. QM9~\citep{QM9}, ZINC~\citep{ZINC}, and ogbg-molhiv are three datasets of molecules. QM9 use 13 quantum chemistry property as the graph label. It provides both the graph and the coordinates of each atom. ZINC provides graph structure only and aim to predict constrained solubility. Ogbg-molhiv is one of Open Graph Benchmark dataset, which aims to use graph structure to predict whether a molecule can inhibits HIV virus replication. We further use MUTAG, PTC-MR, PROTEINS, and IMDB-BINARY from TU database~\citep{TU}. MUTAG comprises 188 mutagenic aromatic and heteroaromatic nitro compounds. PROTEINS represents secondary structure elements as nodes with edges between neighbors in amino-acid sequence or 3D space. PTC involves 344 chemical compounds, classifying carcinogenicity for rats. IMDB-BINARY features ego-networks for actors/actresses in movie collaborations, classifying movie genre graphs. We also use three datasets in Long Range Graph Benchmark~\citep{LRGB}. They consists of larger graphs. PascalVOC-SP comes from the computer vision domain. Each node in it representation a superpixel and the task is to predict the semantic segmentation label for each node. Peptide-func and peptide struct are peptide molecular graphs. Task in Peptides-func is to predict the peptide function. Peptides-struct is to predict 3D properties of the peptide. PTC is a collection of 344 chemical compounds represented as graphs which report the carcinogenicity for rats. There are 19 node labels for each node.

\section{Experimental Details}\label{app::exp}

Our code is available in supplementary material. Our code is primarily based on PyTorch~\citep{Pytorch} and PyG~\citep{PyG}. All our experiments are conducted on NVIDIA RTX 3090 GPUs on a linux server. We use l1 loss for regression tasks and cross entropy loss for classification tasks. We select the hyperparameters by running optuna~\citep{optuna} to optimize the validation score. We run each experiment with 8 different seeds, reporting the averaged results at the epoch achieving the best validation metric. For optimization, we use AdamW optimizer and cosine annealing scheduler. Hyperparameters for datasets are shown in Table~\ref{tab::hyperparam}. All PST and PSDS models (except these in ablation study) decompose laplacian matrix for coordinates.

\begin{table}[t]
    \centering
    \caption{Hyperparameters of our model for each dataset. \#warm means the number of warmup epochs, \#cos denotes the number of cosine annealing epochs, gn denotes the magnitude of the gaussian noise injected into the point coordinates, hiddim denotes hidden dimension, bs means batch size, lr represents learning rate, and \#epoch is the number of epochs for training. }
    \label{tab::hyperparam}
    \begin{small}
    \begin{tabular}{lcccccccccc}
    \toprule
        dataset &\#warm & \#cos & wd & gn & \#layer & hiddim & bs & lr &\#epoch & \#param\\
    \midrule
        Synthetic& 10 & 15 & 6e-4 & 1e-6 &9&96&16&0.0006&300 & 961k \\
        qm9 & 1 & 40 &1e-1 & 1e-5& 8& 128 &256&0.001&150 & 1587k\\
        ZINC& 17 & 17 & 1e-1 & 1e-4& 6 & 80 &128&0.001&800 & 472k\\
        ZINC-full& 40 & 40 & 1e-1 & 1e-6& 8&80 &512&0.003 & 400& 582k\\
        ogbg-molhiv& 5 & 5&1e-1&1e-6& 6&96&24&0.001&300 & 751k\\
        MUTAG & 20 & 1 & 1e-7 & 1e-4 & 2 & 48 & 64 &2e-3 &70 & 82k\\
        PTC-MR & 25 & 1 & 1e-1 & 1e-3 & 4 & 16 & 64 &3e-3 &70 & 15k\\
        PROTEINS & 25 & 1 & 1e-7 & 3e-3 & 2 & 48 & 8 &1.5e-3 &80 & 82k\\
        IMDB-BINARY & 35 & 1 & 1e-7 & 1e-5 & 3 & 48 & 64 &3e-3 &80 & 100k\\
        PascalVOC-SP & 5&5 & 1e-1&1e-5&4&96&6&0.001&40 & 527k\\
        Peptide-func & 40 & 20 & 1e-1 & 1e-6&6&128&2&0.0003&80& 1337k\\
        Peptide-struct &40 &20 &1e-1&1e-6&6&128&2&0.0003&40& 1337k \\
    \bottomrule
    \end{tabular}
    \end{small}
\end{table}

ZINC, ZINC-full, PascalVOC-SP, Peptide-func, and Peptide-struct have 500k parameter budgets. Other datasets have no parameter limit. Graphormer~\citep{GT-graphormer} takes 47000k parameters on ogbg-molhiv. 1-2-3-GNN takes 929k parameters on qm9. Our PST follows these budgets on ZINC, ZINC-full and PascalVOC-SP. However, on the peptide-func and peptide-struct datasets, we find that hidden dimension is quite crucial for good performance. So we use a comparable number of hidden dimension and transformer layers. This leads to about 1M parameters because our PST employs two sets of parameters (one for scalar and one for vector), which resulted in twice the parameters with the same hidden dimension and number of transformer layers. We conduct experiments with baselines with larger hidden dimensions for these datasets. The results are shown in Table~\ref{tab:pep1M}. When the PST and baselines are all to 1M parameters, out PST outperforms baselines with the same parameter budget 1M, and our method is effective on the two datasets.

\begin{table}[t]
\caption{Results on peptide-func and peptide-struct dataset with 1M parameter budget.}\label{tab:pep1M}
    \centering
    \begin{tabular}{lcc}
    \toprule
         & peptide-func & peptide-struct \\
    \midrule
     Graph-MLPMixer	    & $0.6970\pm 0.0080$ & $0.2475\pm 0.0015$\\
     GraphGPS	    & $0.6535\pm 0.0041$& $0.2500\pm 0.0005$\\
     PST & $0.6984\pm 0.0051$ & $0.2470\pm 0.0015$\\
    \bottomrule
    \end{tabular}
\end{table}
Other datasets we explored do not have explicit parameter constraints, and it's worth noting that our PST has fewer parameters compared to representative baselines in these cases.
Hyperparameter Configuration:

Our experiments involved tuning seven hyperparameters: depth (4, 8), hidden dimension (64, 256), learning rate (0.0001, 0.003), number of warmup epochs (0, 40), number of cosine annealing epochs (1, 20), magnitude of Gaussian noise added to input (1e-6, 1e-4), and weight decay (1e-6, 1e-1).
We observed that [1] also used seven hyperparameters in their setup.
Batch size was determined based on GPU memory constraints and was not a parameter that we optimized.

\section{Architecture}\label{app::arch}
\begin{figure}[t]
    \centering
    \includegraphics[width=\textwidth]{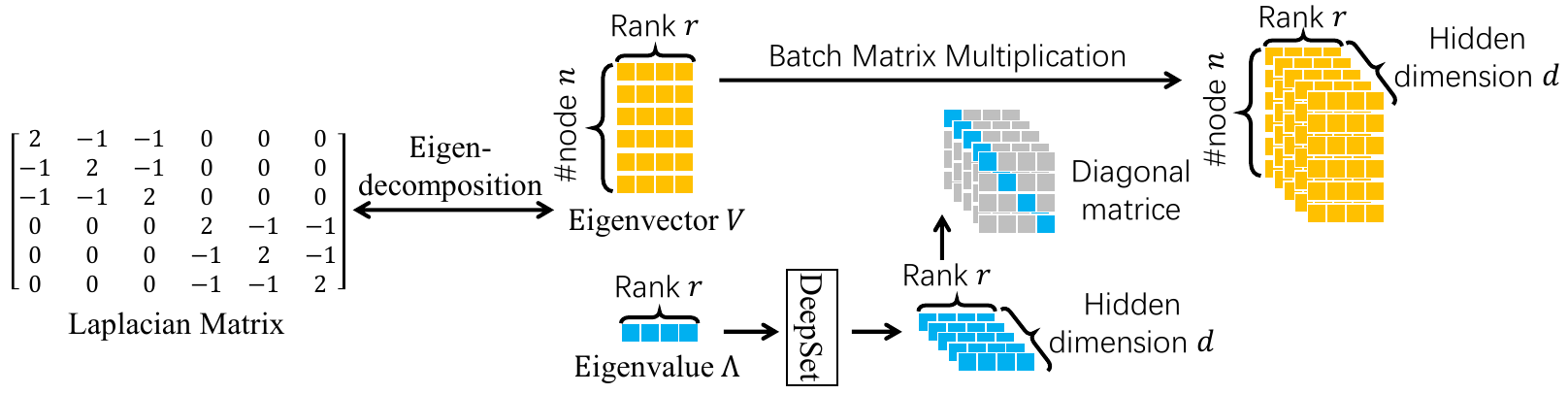}
    \vskip -10pt
    \caption{The pipeline of parameterized SRD. We first decompose Laplacian matrix or other matrice for the non-zero eigenvalue and the corresponding eigenvectors. Then the eigenvalue is transformed with DeepSet~\citep{DeepSet}. Multiply the transformed eigenvalue and the eigenvector leads to coordinates.}
    \label{fig:g2sarch}
    \vskip -5pt
\end{figure}

\begin{figure}[t]
    \centering
    \includegraphics[width=\textwidth]{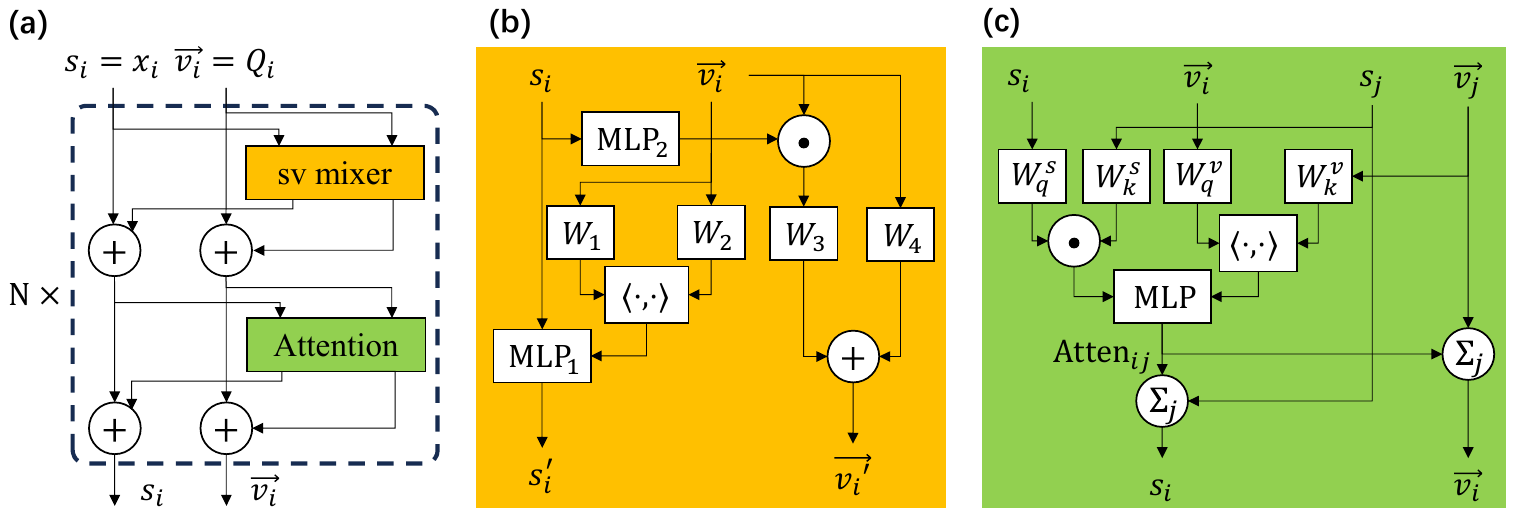}
    \vskip -10pt
    \caption{Architecture of Point Set Transformer (PST) (a) PST contains several layers. Each layer is composed of an scalar-vector (sv)-mixer and an attention layer. (b) The architecture of sv-mixer. (c) The architecture of attention layer. $s_i$ and $s_i'$ denote the scalar representations of node $i$, and $\vec v_i$ and $\vec v_i'$ denote the vector representations. $x_i$ is the initial   features of node $i$. $Q_i$ and point coordinates of node $i$ produced by parameterized SRD in Section~\ref{sec::psrd}.}
    \label{fig:arch}
    \vskip -5pt
\end{figure}

The architecture of parameterized SRD is shown in Figure~\ref{fig:g2sarch}. As illustrated in Section~\ref{sec::psrd}, it first do eigendecomposition for non-zero eigenvalues and the corresponding eigenvectors, then use DeepSet~\citep{DeepSet} to process the eigenvalues, leading to coordinates with multiple channels. The architecture of PST is shown in Figure~\ref{fig:arch}. As illustrated in Section~\ref{sec:pst}, it is composed of scalar-vector mixers and attention layers. 

\section{Ablation}\label{app::abl}
To assess the design choices made in our Point Set Transformer (PST), we conducted ablation experiments. First, we replace the PSRD coordinates (see Section~\ref{sec::psrd}) with SRD coordinates, resulting in a reduced version referred to as the PST-gc model. Additionally, we introduced a variant called PST-onelayer, which is distinct from PST in that it only computes the attention matrix once and does not combine information in scalar and vector features. { 
Furthermore, PST decompose Laplacian matrix by default to produce coordinates. PST-adj uses adjacency matrix instead. Similar to PST, PSDS takes node coordinates as input. However, it use DeepSet~\citep{DeepSet} rather than transformer as the set encoder. For better comparison, we also use our strongest baseline on QM9 dataset, DF~\citep{DRFWL}. }

The results of the ablation study conducted on the QM9 dataset are summarized in Table~\ref{tab::qm9}. Notably, PST-gc exhibits only a slight increase in test loss compared to PST, and even outperforms PST on $4$ out of $12$ target metrics, highlighting the effectiveness of the Graph as Point Set approach with vanilla symmetric rank decomposition. In contrast, PST-onelayer performs significantly worse, underscoring the advantages of PST over previous methods that augment adjacency matrices with spectral features. { PST-adj and PST-normadj achieves similar performance to PST, illustrating that the choice of matrix to decompose does not matter. DeepSet performs worse than PST, but it still outperforms our strongest baseline DF, showing the potential of combining set encoders other than transformer with our convertion from graph to set.} On the long-range graph benchmark, PST maintains a significant performance edge over PST-onelayer. However, it's worth noting that the gap between PST and PST-gc widens, further confirming the effectiveness of gc in modeling long-range interactions.

\begin{table}
\caption{Ablation study on qm9 dataset.}\label{tab:qm9abl}
    \centering
    \setlength{\tabcolsep}{1pt}
    \resizebox{\linewidth}{!}{
    \begin{small}
    \begin{tabular}{ccccccccccccc}
    \toprule
        ~&$\mu$ & $\alpha$ & $\varepsilon_{\text{homo}}$ & $\varepsilon_{\text{lumo}}$ & $\Delta\varepsilon$ & $R^2$ & ZPVE & $U_0$ & $U$ & $H$ & $G$ & $C_v$ \\ 
        \midrule
        Unit & $10^{-1}$D  & $10^{-1}$$a_0^3$  & $10^{-2}$meV  & $10^{-2}$meV  & $10^{-2}$meV  & $a_0^2$  & $10^{-2}$meV  & meV  & meV  & meV  & meV  & $10^{-2}$cal/mol/K \\
        \midrule
        PST & $3.19_{\pm 0.04}$  & $1.89_{\pm 0.04}$  & $5.98_{\pm 0.09}$  & $5.84_{\pm 0.08}$  & $8.46_{\pm 0.07}$  & $13.08_{\pm 0.16}$  & $0.39_{\pm 0.01}$  & $3.46_{\pm 0.17}$  & $3.55_{\pm 0.10}$  & $3.49_{\pm 0.20}$  & $3.55_{\pm 0.17}$  & $7.77_{\pm 0.15}$  \\ 
        PST-onelayer & $3.72_{\pm 0.02}$  & $2.25_{\pm 0.05}$  & $6.62_{\pm 0.11}$  & $6.67_{\pm 0.07}$  & $9.37_{\pm 0.15}$  & $15.95_{\pm 0.29}$  & $0.55_{\pm 0.01}$  & $3.46_{\pm 0.06}$  & $3.50_{\pm 0.14}$  & $3.50_{\pm 0.03}$  & $3.45_{\pm 0.07}$  & $9.62_{\pm 0.24}$  \\ 
         PST-gc & $3.34_{\pm 0.02}$  & $1.93_{\pm 0.03}$  & $6.08_{\pm 0.11}$  & $6.10_{\pm 0.10}$  & $8.65_{\pm 0.10}$  & $13.71_{\pm 0.12}$  & $0.40_{\pm 0.01}$  & $3.38_{\pm 0.13}$  & $3.43_{\pm 0.12}$  & $3.33_{\pm 0.08}$  & $3.29_{\pm 0.11}$  & $8.04_{\pm 0.15}$ \\
         PST-adj & $3.16_{\pm 0.02}$ & $1.86_{\pm 0.01}$ & $6.31_{\pm 0.06}$ & $6.10_{\pm 0.05}$ & $8.84_{\pm 0.01}$ & $13.60_{\pm 0.09}$ & $0.39_{\pm 0.01}$ & $3.59_{\pm 0.12}$ & $3.73_{\pm 0.08}$ & $3.65_{\pm 0.06}$ & $3.60_{\pm 0.016}$ & $7.62_{\pm 0.21}$ \\
         PST-normadj & $3.22_{\pm 0.04}$ & $1.85_{\pm 0.02}$ & $5.97_{\pm 0.23}$ & $6.15_{\pm 0.07}$ & $8.79_{\pm 0.04}$ & $13.42_{\pm 0.15}$ & $0.41_{\pm 0.01}$ & $3.36_{\pm 0.25}$ & $3.41_{\pm 0.24}$ & $3.46_{\pm 0.18}$ & $3.38_{\pm 0.23}$ & $8.10_{\pm 0.12}$\\
         PSDS & $3.53_{\pm 0.05}$ & $2.05_{\pm 0.02}$ & $6.56_{\pm 0.03}$ & $6.31_{\pm 0.05}$ & $9.13_{\pm 0.04}$ & $14.35_{\pm 0.02}$ & $0.41_{\pm 0.02}$ & $3.53_{\pm 0.11}$ & $3.49_{\pm 0.05}$ & $3.47_{\pm 0.04}$ & $3.56_{\pm 0.14}$ & $8.35_{\pm 0.09}$ \\
         DF & $3.46$ & $2.22$ & $6.15$ & $6.12$ & $8.82$ & $15.04$ & $0.46$ & $4.24$ & $4.16$ & $3.95$ & $4.24$ & $9.01$\\
         \bottomrule
    \end{tabular}
    \end{small}
    }
\end{table}

\begin{table}[t]
\centering
\caption{Ablation study on Long Range Graph Benchmark dataset.}\label{tab:LRGBabl}
    \setlength{\tabcolsep}{1pt}
\begin{small}
    \begin{tabular}{lccc}
    \toprule
        Model & PascalVOC-SP & Peptides-Func & Peptides-Struct \\ 
    \midrule
        PST & $0.4010_{\pm0.0072}$ &  $0.6984_{\pm0.0051}$ & $0.2470_{\pm 0.0015}$\\

        PST-onelayer & $0.3229_{\pm0.0051}$ &  $0.6517_{\pm0.0076}$ & $0.2634_{\pm 0.0019}$\\
        PST-gc &$0.4007_{\pm0.0039}$  &  $0.6439_{\pm0.0342}$ & $0.2564_{\pm 0.0120}$\\
        \bottomrule
    \end{tabular}
\end{small}
\end{table}

\section{Scalability}\label{app::scalability}
We present training time per epoch and GPU memory consumption data in Table~\ref{tab:scalability} and Table~\ref{tab:scalability_pascalvoc}. Due to architecture, PST has higher time complexity than existing Graph Transformers and does not scale well on large graphs like pascalvoc-sp dataset. However, on the ZINC dataset, PST ranks as the second fastest model, and its memory consumption is comparable to existing models with strong expressivity, such as SUN and SSWL, and notably lower than PPGN.  

\begin{table}[t]
\caption{Training time per epoch and GPU memory consumption on zinc dataset with batch size 128.}\label{tab:scalability}
    \centering
    \begin{tabular}{cccccccc}
    \toprule
        ~ & PST&SUN & SSWL & PPGN & Graphormer & GPS & SAN-GPS \\ 
    \midrule
        Time/s & 15.20& 20.93  & 45.30  & 20.21    & 123.79  & 11.70  & 79.08  \\ 
        Memory/GB& 4.08 & 3.72  & 3.89  & 20.37    & 0.07  & 0.25  & 2.00 \\ 
    \bottomrule
    \end{tabular}
\end{table}

\begin{table}[t]
\caption{Training time per epoch and GPU memory consumption on pascalvoc-sp dataset with batch size 6.}\label{tab:scalability_pascalvoc}
    \centering
    \begin{tabular}{cccccccc}
    \toprule
        ~ & PST&SUN & SSWL & PPGN & Graphormer & GPS & SAN-GPS \\ 
    \midrule
        Time/s & 15.20& 20.93  & 45.30  & 20.21    & 123.79  & 11.70  & 79.08  \\ 
        Memory/GB& 4.08 & 3.72  & 3.89  & 20.37    & 0.07  & 0.25  & 2.00 \\ 
    \bottomrule
    \end{tabular}
\end{table}

\section{Results on TU datasets}\label{app::exp_TU}
Following the setting of ~\citet{KPGNN}, we test our PST on four TU datasets~\citep{TU}. The results are shown in Table~\ref{tab:tu}. Baselines include WL subtree kernel~\citep{WLsubtreeKernel}, GIN~\citep{GIN}, DGCNN~\citep{DGCNN}, GraphSNN~\citep{GraphSNN}, GNN-AK+~\citep{GNNAK}, and three variants of KP-GNN~\citep{KPGNN} (KP-GCN, KP-GraphSAGE, and KP-GIN). We use 10-fold cross-validation, where 9 folds are for training and 1 fold is for testing. The average test accuracy is reported. Our PST consistently outperforms our baselines.
\begin{table*}[th]
\vspace{-5pt}
\begin{center}

\caption{TU dataset evaluation result.}
\label{tab:tu}
  \begin{tabular}{lcccc}
   \toprule
    Method & MUTAG & PTC-MR & PROTEINS & IMDB-B\\
    \midrule  
    \textbf{WL} & 90.4{\small$\pm$5.7}& 59.9{\small$\pm$4.3} &75.0{\small$\pm$3.1} & 73.8{\small$\pm$3.9}  \\
    \midrule
    \textbf{GIN} & 89.4{\small$\pm$5.6}  & 64.6{\small$\pm$7.0} & 75.9{\small$\pm$2.8} & 75.1{\small$\pm$5.1}\\
    \textbf{DGCNN} & 85.8{\small$\pm$1.7}  & 58.6 {\small$\pm$2.5} & 75.5{\small$\pm$0.9} & 70.0{\small$\pm$0.9} \\
    \midrule
    \textbf{GraphSNN} & 91.24{\small$\pm$2.5}  & 66.96{\small$\pm$3.5} & 76.51{\small$\pm$2.5} & 76.93{\small$\pm$3.3}\\
    \textbf{GIN-AK+} & 91.30{\small$\pm$7.0}  & 68.20{\small$\pm$5.6} & 77.10{\small$\pm$5.7} & 75.60{\small$\pm$3.7}\\
    \midrule
    \textbf{KP-GCN} & 91.7{\small$\pm$6.0}  & 67.1{\small$\pm$6.3} & 75.8{\small$\pm$3.5} & 75.9{\small$\pm$3.8}\\
    \textbf{KP-GraphSAGE} & 91.7{\small$\pm$6.5}  & 66.5{\small$\pm$4.0} & 76.5{\small$\pm$4.6} & 76.4{\small$\pm$2.7} \\
    \textbf{KP-GIN} & 92.2{\small$\pm$6.5} & 66.8{\small$\pm$6.8} & 75.8{\small$\pm$4.6} &  76.6{\small$\pm$4.2}\\
    \midrule
  \textbf{PST} & \textbf{94.4{\small$\pm$3.5}}  & \textbf{68.8{\small$\pm$4.6}} & \textbf{80.7{\small$\pm$3.5}} &  \textbf{78.9{\small$\pm$3.6}} \\
  \bottomrule
\end{tabular}

\end{center}
\vspace{-10pt}
\end{table*}

\section{Point Set DeepSet}\label{app::PSDS}
Besides Transformer, we also propose a DeepSet~\citep{DeepSet}-Based set encoder, Point Set DeepSet (PSDS), for point set to illustrate the versatility of our graph-to-set method. Similar to PST, PSDS also operates with points carrying two types of representations: scalars, which remain invariant to coordinate orthogonal transformations, and vectors, which adapt equivariantly to coordinate changes. For a point $i$, its scalar representation is denoted by $s_i\in \sR^{d}$, and its vector representation is denoted by $v_i\in \sR^{r\times d}$, where $d$ is the hidden dimension, and $r$ is the rank of coordinates. $s_i$ is initialized with the input node feature $X_i$, and $v_i$ is initialized with the parameterized coordinates containing graph structural information, as detailed in Section~\ref{sec::psrd}. Similar to DeepSet, PSDS transforms point representations individually, aggregates them to produce global feature, and combine global features and individual point representations to update point representations.

\textbf{Scalar-Vector Mixer.}~This component individually transforms point representations. To enable the information exchange between vector and scalar features, we design a mixer architecture as follows.
\begin{align}
s_i' &\leftarrow \text{MLP}_1(s_i \Vert \text{diagonal}(W_1v_i^Tv_iW_2^T)), \\
v_i' &\leftarrow v_i \text{diag}(\text{MLP}_2(s_i))W_3 + v_iW_4
\end{align}
Here, $W_1, W_2, W_3,$ and $W_4 \in \mathbb{R}^{d\times d}$ are learnable matrices for mixing different channels of vector features. Additionally, $\text{MLP}_1:\mathbb{R}^{2d\to d}$ and $\text{MLP}_2:\mathbb{R}^{d\to d}$ represent two multi-layer perceptrons transforming scalar representations. The operation $\text{diagonal}(W_1v_i^Tv_iW_2)$ takes the diagonal elements of a matrix, which translates vectors to scalars, while $\text{diag}(\text{MLP}_2(s_i))v_i$ transforms scalar features into vectors. As $v_i^TR^TR v_i = v_i^T v_i, \forall R \in O(r)$, the scalar update is invariant to orthogonal transformations of the coordinates. Similarly, the vector update is equivariant to $O(r)$.

\textbf{Aggregator.}~This component aggregates individual point representations for global features $s, v, vsq$. 
\begin{small}
\begin{align}
s &\leftarrow \sum_{i\in V}\text{MLP}_3(s_i)\\
v &\leftarrow \sum_{i\in V}v_i W_5\\
v_{sq}&\leftarrow \sum_{i\in V}v_iW_6W_7v_i^T
\end{align}
\end{small}

Here, $W_5, W_6$ and $W_7\in \sR^{d\times d}$ denote the linear transformations vectors. $\text{MLP}_3: \sR^d\to \sR^d$ is an MLP converting scalars. Global feature $s\in \sR^d$ is scalar, $v\sR^{r\times d}$ is vector, and $v_{sq}\in \sR^{r\times r}$ is the sum of square for each vector. 

\textbf{Point Representation Update.}~Each point representation is updated by combining global features.
\begin{small}
\begin{align}
s_i &\leftarrow \text{MLP}_4(s_i+s)\\
v_i &\leftarrow v_{sq}v_i + vW_8
\end{align}
\end{small}
$s_i$ is combined with global scalar $s$ and transformed with an MLP $\text{MLP}_4: \sR^d\to \sR^d$. $v_i$ is combined with $v_{sq}$ and $v$ with linear layer $W_8\in \sR^{d\times d}$.

\textbf{Pooling.}~After several layers, we pool all points' scalar representations as the set representation $s$.
\begin{equation}
    s\leftarrow\text{Pool}(\{s_i|i\in V\}),
\end{equation}
where $\text{Pool}$ is pooling function like sum, mean, and max.

\end{document}